\documentclass[sigconf]{lib/acmart}

\usepackage[utf8]{inputenc} 
\usepackage[T1]{fontenc}    
\usepackage{hyperref}       
\usepackage{url}            
\usepackage{booktabs}       
\usepackage{amsfonts}       
\usepackage{nicefrac}       
\usepackage{microtype}      
\usepackage{xcolor}         

\usepackage{graphicx}
\usepackage{array}
\usepackage{latexsym}
\usepackage{amsmath}
\usepackage{amsthm}
\usepackage{mathtools}
\usepackage{bm}
\usepackage{xspace}
\usepackage{physics}
\usepackage{colortbl}
\usepackage{algorithm}
\usepackage{algpseudocode}
\usepackage{url}
\usepackage{multirow}
\usepackage{multicol}
\usepackage{enumitem}
\usepackage{pifont}
\usepackage[subrefformat=parens]{subcaption}
\usepackage{xifthen}  
\usepackage{xparse}  
\usepackage{ifthen}  
\usepackage{tikz-cd}
\usetikzlibrary{automata, positioning}
\usepackage{leftidx}
\usepackage{tcolorbox}
\usepackage{lib/original}
\usepackage{tabularx}
\usepackage{cleveref}
\usepackage{etoc}

\NewDocumentCommand{\lookback}{m o}{%
  {\textcolor[HTML]{ff0000}{#1}}%
  \IfValueT{#2}{%
    \noindent{\color{blue}~(Note: #2)}%
  }%
}

\graphicspath{{./fig/}}

\allowdisplaybreaks[4]

\crefname{figure}{Figure}{Figures}
\crefname{equation}{Eq.}{Eqs.}
\crefname{model}{Model}{Models}
\crefname{section}{Section}{Sections}
\crefname{algorithm}{Algorithm}{Algorithms}
\crefname{assumption}{Assumption}{Assumptions}
\crefname{table}{Table}{Tables}
\crefname{lemma}{Lemma}{Lemmas}
\crefname{definition}{Definition}{Definitions}
\crefname{appendix}{Appendix}{Appendixes}

\newcommand{\argmax}{\mathop{\rm arg\,max}}
\newcommand{\argmin}{\mathop{\rm arg\,min}}
\NewDocumentCommand{\Exp}{ m o }{%
\mathbb{E}\IfValueT{#2}{_{#2}}\!\left[#1\right]%
}

\newcommand{\idmat}{\mI}
\newcommand{\R}{\mathbb{R}}
\newcommand{\C}{\mathbb{C}}
\newcommand{\N}{\mathbb{N}}

\newcommand{\E}{\mathbb{E}}

\newcommand{\inner}[2]{\left\langle#1, #2\right\rangle}

\renewcommand{\mid}{\,|\,}

\newcommand{\mypara}[1]{\noindent\textbf{#1.}}

\newcommand{\myparaitemize}[1]{\noindent{\textbf{#1.}}}

\newtheorem{problem}{Problem}

\newtheorem{definition}{Definition}

\newtheorem{model}{Model}

\newcommand{\rand}[1]{{#1}}

\def\randx{{\rand{{X}}}}

\newcommand{\mat}[1]{\mathbf{#1}}

\newcommand{\vect}[1]{\mathbf{#1}}

\def\mA{{\mat{A}}}

\def\mC{{\mat{C}}}

\def\mF{{\mat{F}}}
\def\mG{{\mat{G}}}
\def\mH{{\mat{H}}}
\def\mI{{\mat{I}}}
\def\mJ{{\mat{J}}}
\def\mK{{\mat{K}}}

\def\mM{{\mat{M}}}

\def\mP{{\mat{P}}}
\def\mQ{{\mat{Q}}}
\def\mR{{\mat{R}}}
\def\mS{{\mat{S}}}

\def\mU{{\mat{U}}}
\def\mV{{\mat{V}}}
\def\mW{{\mat{W}}}
\def\mX{{\mat{X}}}

\def\mZ{{\mat{Z}}}

\def\va{{\vect{a}}}

\def\vc{{\vect{c}}}

\def\vk{{\vect{k}}}

\def\vx{{\vect{x}}}
\def\vy{{\vect{y}}}
\def\vz{{\vect{z}}}

\newcommand{\cC}{\mathcal{C}}

\newcommand{\cS}{\mathcal{S}}

\DeclareMathAlphabet{\mathsfit}{\encodingdefault}{\sfdefault}{m}{sl}
\SetMathAlphabet{\mathsfit}{bold}{\encodingdefault}{\sfdefault}{bx}{n}

\newcommand{\stream}{\mat{X}}

\renewcommand{\t}{t}
\renewcommand{\tt}{{\t+1}}
\newcommand{\tindex}[1]{{_{#1}}}
\newcommand{\ut}{\tindex{t}}
\newcommand{\utt}{\tindex{t+1}}
\newcommand{\tc}{{t_c}}
\newcommand{\ts}{{t_s}}
\newcommand{\obs}{\vx}

\newcommand{\obsone}{\vect{x}\tindex{1}}
\newcommand{\obst}{\vect{x}\ut}
\newcommand{\obstt}{\vect{x}\utt}
\newcommand{\obstc}{\vect{x}\tindex{\tc}}

\renewcommand{\d}{d}
\newcommand{\n}{T}
\newcommand{\ls}{\ell_s}

\newcommand{\Hhil}{\mathcal{H}}

\newcommand{\Hnorm}[1]{\norm{#1}_{\Hhil}}

\newcommand{\Hinner}[2]{\inner{#1}{#2}_\Hhil}

\newcommand{\fmap}{\phi}

\newcommand{\dynamics}{\bm{f}}
\newcommand{\noise}{\boldsymbol\omega}

\newcommand{\M}{\mathcal{M}}
\newcommand{\F}{\mathcal{F}}

\newcommand{\koop}{\mathcal{K}_\pi}

\newcommand{\obssca}{\eta}

\newcommand{\invm}{\pi}
\newcommand{\Ltwo}{L_\pi^2(\M)}
\newcommand{\eigf}{\psi}

\newcommand{\eigv}{\lambda}

\newcommand{\mode}{\gamma}

\renewcommand{\rank}{r}

\newcommand{\latent}{\vz}
\newcommand{\est}{\hat{\vx}}

\newcommand{\indices}{\mathcal{I}}

\newcommand{\regime}{\boldsymbol{\theta}}
\newcommand{\regimeset}{\boldsymbol{\Theta}}
\newcommand{\nrgm}{R}

\newcommand{\candparam}{\mathcal{C}}

\newcommand{\dict}{\mathcal{D}}

\newcommand{\initcond}{\boldsymbol\mu}

\newcommand{\initcov}{\mat{P}}

\newcommand{\ith}[2][]{%
    \ifthenelse{\isempty{#1}}%
    {
        #2_{(i)}
    }%
    {
        {#2_{(i)}^{#1}}
    }%
}
\newcommand{\first}[2][]{%
    \ifthenelse{\isempty{#1}}%
    {
        #2_{(1)}
    }%
    {
        {#2_{(1)}^{#1}}
    }%
}
\newcommand{\dth}[2][]{%
    \ifthenelse{\isempty{#1}}%
    {
        #2_{(d)}
    }%
    {
        {#2_{(d)}^{#1}}
    }%
}

\newcommand{\methodname}{\textsc{AdaKoop}}
\newcommand{\method}{\textsc{\methodname}\xspace}

\definecolor{ForestGreen}{rgb}{0.13, 0.55, 0.13} 
\definecolor{BrickRed}{rgb}{0.80, 0.25, 0.33}    
\newcommand*\rot{\rotatebox{90}}
\newcommand*\OK{\ding{51}}
\newcommand*\NO{}

\AtBeginDocument{%
  }

\copyrightyear{2026}
\acmYear{2026}
\setcopyright{cc}
\setcctype{by}
\acmDOI{10.1145/3770855.3817851}
\acmConference[KDD '26]{Proceedings of the 32nd ACM SIGKDD Conference on Knowledge Discovery and Data Mining V.2}{August 09--13, 2026}{Jeju Island, Republic of Korea}
\acmBooktitle{Proceedings of the 32nd ACM SIGKDD Conference on Knowledge Discovery and Data Mining V.2 (KDD '26), August 09--13, 2026, Jeju Island, Republic of Korea}
\acmISBN{979-8-4007-2259-2/2026/08}

\newtheorem{lemma}{Lemma}
\begin{document}





\title{
\textsc{AdaKoop}:
Efficient Modeling of Nonlinear Dynamics from Nonstationary Data Streams with Koopman Operator Regression}


\renewcommand{\shorttitle}{
Efficient Modeling of Nonlinear Dynamics from Nonstationary Data Streams}

\author{Naoki Chihara}
\affiliation{%
  \institution{SANKEN, The University of Osaka}
  \city{Osaka}
  \country{Japan}
}
\email{naoki88@sanken.osaka-u.ac.jp}

\author{Ren Fujiwara}
\affiliation{%
  \institution{SANKEN, The University of Osaka}
  \city{Osaka}
  \country{Japan}
}
\email{r-fujiwr@sanken.osaka-u.ac.jp}

\author{Yasuko Matsubara}
\affiliation{%
  \institution{SANKEN, The University of Osaka}
  \city{Osaka}
  \country{Japan}
}
\email{yasuko@sanken.osaka-u.ac.jp}

\author{Yasushi Sakurai}
\affiliation{%
  \institution{SANKEN, The University of Osaka}
  \city{Osaka}
  \country{Japan}
}
\email{yasushi@sanken.osaka-u.ac.jp}



\begin{abstract}
Real-time data analysis requires the ability to accurately and adaptively address nonlinear dynamics in a nonstationary data stream while preserving computational efficiency.
However, nonlinear dynamics are so complex that capturing dynamically changing nonlinear patterns and utilizing them for downstream tasks under strict time constraints is nontrivial.
To bridge the gap between nonlinear complexity and computational tractability, this study applies Koopman operator theory, which states that nonlinear dynamics can be represented as linear transitions in an infinite-dimensional space.
Building upon finite-dimensional approximations of this operator,
we present \method, an efficient streaming algorithm for modeling nonlinear dynamics over nonstationary data streams.
Our approach utilizes a probabilistic framework grounded in Koopman operator theory, treating both raw observations and reproducing kernel Hilbert space (RKHS) features as emissions from latent vectors.
This dual-view formulation allows nonlinear dynamics to be expressed as a tractable linear system.
Therefore, \method enables the efficient and stable modeling of nonlinear dynamics in a streaming fashion, avoiding the prohibitive computational costs of iterative nonlinear optimization.
Furthermore, to address nonstationarity in data streams,
\method adaptively detects the switching of patterns via statistical hypothesis testing for abrupt pattern shifts and incrementally updates model parameters to handle continuous changes.
Extensive experiments on a total of 71 practical benchmark datasets across various domains demonstrate that \method outperforms state-of-the-art methods in terms of real-time forecasting accuracy and computational efficiency.

\end{abstract}

\maketitle

\section{Introduction}
    \label{section:introduction}
    With the rapid worldwide digitalization, massive volumes of time series data are collected and processed sequentially, calling for efficient and precise stream processing capabilities across various applications, such as precipitation nowcasting~\cite{Foresti2016-gf} and real-time wearable health monitoring~\cite{Kalantarian2017-yq}.
Unlike batch learning settings, where all the observed data is available in advance, these environments involve potentially semi-infinite data streams.
The unbounded nature of these data streams makes it impossible to store the entire history, particularly in resource-constrained edge computing environments~\cite{Matsubara2025-fv}. This requires performing downstream tasks (e.g., real-time forecasting) in a single pass over incoming data while adapting rapidly to current dynamic patterns.
\par
However, achieving such adaptive processing remains a challenge because real-world data streams often exhibit intrinsic complexities that hinder efficient real-time data analysis under strict latency constraints.
Specifically, we mainly encounter two challenges:
\mypara{(a) Nonlinear dynamics}
Real-world data streams inherently exhibit complex nonlinear dynamics~\cite{Dysts}.
For example, atmospheric weather patterns are governed by the Lorenz system~\cite{lorenz}, which is a typical nonlinear dynamical system.
In healthcare, variability patterns observed in physiological signals, such as the electrocardiogram (ECG), exhibit deterministic chaotic behavior~\cite{Gupta2019-qs}.
Nonlinear dynamics are too complex for simple linear models, so a more expressive model is required.
Recently, the data-driven discovery of governing nonlinear dynamics from data has attracted significant interest, and various approaches have been proposed, including polynomial regressions~\cite{brunton2016discovering, bertsimas2023learning, fujiwara2025modeling} and
recurrent neural networks (RNNs)~\cite{cao2018brits, chen2018neural}.
However, most previous methods involve computationally intensive optimization or exhibit high sensitivity to noise, making them suboptimal for data stream processing.
\\
\mypara{(b) Nonstationarity}
The statistical properties of a data stream can change continuously and abruptly over time due to external influences or shifts in underlying mechanisms.
For example, in wearable health monitoring, physiological signal patterns may change gradually with circadian rhythms but can shift abruptly when the wearer changes activities (e.g., from resting to exercising) or when the sensor is re-positioned.
This variability further exacerbates data stream processing because future values may exhibit patterns different from those observed in the past.
Specifically, even if a model can learn patterns accurately in the early stages, it will inevitably degrade as time-changing patterns emerge.
Based on the above discussion, our problem is summarized as follows:
\begin{tcolorbox}[colback=gray!10, colframe=gray!10, sharp corners=all]
\textit{How can we estimate complex nonlinear dynamics while remaining computationally efficient enough to quickly adapt to both continuous and abrupt changes in a nonstationary data stream?}
\end{tcolorbox}

In this paper, we propose an efficient streaming algorithm called \method, designed for modeling nonlinear dynamics in nonstationary data streams.
To capture nonlinear dynamics efficiently,
we first introduce a probabilistic augmented kernel state space model constructed based on Koopman operator theory, defined on a reproducing kernel Hilbert space (RKHS).
The Koopman operator provides a linear representation of nonlinear dynamical systems by lifting the state space to an infinite-dimensional space.
Our model treats both raw observations and RKHS features as emissions from a shared latent linear dynamical system;
therefore, it enables complex nonlinear dynamics to be expressed as a tractable linear system.
Thus, \method enables efficient and stable modeling of nonlinear dynamics via closed-form updates.
To further improve the efficiency of the estimation, we initialize the model parameters based on Koopman operator regression and subsequently refine them using the expectation-maximization (EM) algorithm.
This initialization serves as a warm start, helping to avoid undesired local optima and accelerate convergence.
In addition, our streaming algorithm performs both abrupt change detection based on statistical hypothesis testing and continuous model parameter updates to adaptively handle nonstationary data streams, making it highly scalable for semi-infinite data streams.

\vspace{0.1em}
\mypara{Contributions}
Our main contributions are summarized as follows:
\begin{itemize}[leftmargin=12.5pt,topsep=1pt]
    \item We introduce a probabilistic dual-view state space model that treats both raw observations and RKHS features as emissions from a shared latent linear system, thereby achieving stable and efficient modeling of nonlinear dynamics.
    \item We present \method, an efficient streaming algorithm with constant time complexity per process, independent of the length of a data stream (see \cref{lemma:time}), that handles time-varying nonlinear dynamics in a data stream. It employs statistical hypothesis testing to adaptively switch models and utilizes closed-form updates for stream processing.
    \item Extensive experiments on 71 practical benchmark datasets show that \method outperforms state-of-the-art methods in both forecasting accuracy and computational efficiency.
\end{itemize}

\mypara{Reproducibility}
Our source code is available at this repository: \url{https://github.com/C-Naoki/AdaKoop}.

\section{Related Work}
    \label{section:related_work}
    \begin{table}[t]
\caption{Capabilities of approaches.}
\centering
\vspace{-0.5em}
\resizebox{1.0\linewidth}{!}{
\begin{tabular}{l|cccc|ccc|cc|c}
\toprule
& \multicolumn{4}{c|}{KOT} & \multicolumn{3}{c|}{TSF} & \multicolumn{2}{c|}{DSP} &  \\
& \rot{FSDMD~\cite{takeishi2018factorially}} & \rot{EDMD~\cite{williams2015data}} & \rot{WDMD~\cite{zhang2019online}} & \rot{Kostic et al.~\cite{kostic2022learning}} & \rot{PAttn~\cite{tan2024language}} & \rot{Koopa~\cite{liu2023koopa}} & \rot{OneNet~\cite{wen2023onenet}} & \rot{sKAF~\cite{giannakis2023learning}} & \rot{ModePlait~\cite{modeplait}} & \rot{\textbf{\method}} \rule[0mm]{0mm}{23mm} \\ 
\midrule
Online algorithm & \NO & \NO & \OK & \NO & \NO & \NO & \OK & \OK & \OK & \OK \rule[0mm]{0mm}{3.1mm} \\ 
Forecasting & \NO & \NO & \NO & \NO & \OK & \OK & \OK & \OK & \OK & \OK \rule[0mm]{0mm}{3.1mm} \\ 
Nonlinearity & \NO & \OK & \NO & \OK & \NO & \OK & \OK & \OK & \OK & \OK \rule[0mm]{0mm}{3.1mm} \\ 
Nonstationarity & \OK & \NO & \OK & \NO & \NO & \OK & \OK & \NO & \OK & \OK \rule[0mm]{0mm}{3.1mm} \\ 
Infinite feature space & \NO & \OK & \NO & \OK & \NO & \NO & \NO & \OK & \NO & \OK \rule[0mm]{0mm}{3.1mm} \\ 
Linear representation & \OK & \OK & \OK & \OK & \NO & \OK & \NO & \NO & \NO & \OK \rule[0mm]{0mm}{3.1mm} \\ 
\bottomrule
    \end{tabular}
}
\normalsize
\label{table:capability}
\vspace{-0.5em}
\end{table}
Here, we provide an overview of investigations related to our work.
Table \ref{table:capability} summarizes six relative advantages of \method.

\myparaitemize{Koopman operator theory}
Koopman operator theory (KOT)~\cite{
koopman1931hamiltonian,
brunton2022koopman} has been developed over decades to analyze and understand nonlinear dynamical systems without any explicit prior knowledge of their governing temporal dynamics~\cite{
kawahara2016dynamic,
kostic2022learning,
berman2023multifactor}.
According to KOT, every deterministic nonlinear dynamical system is represented by a linear yet infinite-dimensional operator without losing information.
One advantage of using linear operators is that they enable spectral decomposition, which provides insights into the behavior of nonlinear dynamical systems.
Dynamic mode decomposition (DMD) is a representative method for approximating the Koopman operator.
DMD was originally proposed as a tool for diagnosing fluid flows~\cite{schmid2010dynamic}, but its reliance on linear monomials as basis functions in the feature space restricts its ability to model nonlinear dynamics.
Subsequent research has explored various extensions to broaden its range of applications~\cite{
proctor2016dynamic,
takeishi2017bayesian}.
For example, FSDMD~\cite{takeishi2018factorially} and Windowed DMD (WDMD)~\cite{zhang2019online} address time-varying linear systems through distinct methodologies.
EDMD~\cite{williams2015data} addresses the limitations of linear monomials by mapping data into a high-dimensional space defined by user-specified nonlinear basis functions.
Recently, several methods aim to learn the Koopman operators on a reproducing kernel Hilbert space (RKHS)~\cite{
klus2020eigendecompositions,
meanti2024estimating,
kostic2024sharp}.
A framework to learn the Koopman operator in an RKHS using finite temporally sampled data was presented in~\cite{kostic2022learning}.
However, they assume that the input data is generated by an autonomous dynamical system, meaning that they cannot handle ever-changing nonlinear dynamics in a nonstationary data stream.

\myparaitemize{Data stream processing}
Considering the real-world scenario where a large number of time series data is generated sequentially,
data stream processing (DSP) has been increasingly recognized for its practicality and efficiency~\cite{gama2019learning}.
In particular, as underlying patterns are likely to shift over time, methods incapable of adapting to these changes often suffer from degraded accuracy.
Thus, DSP has proved greatly significant to the machine learning and database communities~\cite{
engel2004kernel,
aggarwal2007data,
devulapalli2024learning,
Higashiguchi2025-cf,
Nakamura2026-lp,
Kakio2026-kz,
fujiwara2026when}.
ModePlait~\cite{modeplait} is the latest streaming algorithm for discovering time-changing causal relationships and forecasting future values while monitoring transitions of dynamic patterns.
However, it remains challenging to estimate highly nonlinear dynamics.
Although Streaming KAF (sKAF)~\cite{giannakis2023learning} utilizes a streaming algorithm to reduce computational complexity,
its update rule implicitly assumes that each data point is obtained from a static environment.
Moreover, it directly generates future values from current data without focusing on the sequential nature of time series, and its forecasting performance can be unstable.

\myparaitemize{Time series forecasting}
Time series forecasting (TSF) is one of the most important tasks in time series analysis.
Autoregressive integrated moving average (ARIMA)~\cite{box1976arima} and linear dynamical system (LDS)~\cite{Kalman1963-sz} are classical statistical methods for TSF.
In recent years, numerous TSF methods have benefited from deep learning approaches~\cite{
zhou2021informer,
piao2024fredformer,
murad2025wpmixer}.
For example, PAttn~\cite{tan2024language} is a state-of-the-art, simple linear-based method.
Koopa~\cite{liu2023koopa} utilizes Koopman operator theory to learn nonstationary time series.
However, most of these approaches focus only on offline optimization and cannot adaptively handle time-changing patterns in a nonstationary data stream.
In contrast, OneNet~\cite{wen2023onenet} is an online learning approach that addresses transient environmental changes by dynamically updating two complementary models.
While the online ensembling approach decreases the forecasting error, it also increases the number of parameters and computational time, limiting its practical applicability in a streaming setting.
\section{Model Formulation}
    \label{section:model}
    Here, we formulate the problem addressed in this work and then present the \method model designed to tackle this challenge.
\subsection{Problem Definition}
\mypara{Notation}
The main symbols utilized in this paper are in \cref{table:symbols}.
We define $[m:n]=\{m,\ldots,n\}$ and $[n]=\{1,\ldots,n\}$ for $n,m\in\N$.
Let $\M$ be an observation space, and $\Hhil$ be a reproducing kernel Hilbert space (RKHS) with kernel function $k:\M\times\M\to\R$ and feature map $\fmap:\M\to\Hhil$~\cite{aronszajn1950theory}, such that $k(\vx,\vy)=\Hinner{\fmap(\vx)}{\fmap(\vy)}$ for $\vx,\vy\in\M$.
The Moore-Aronszajn theorem guarantees that a symmetric, positive kernel function $k$ provides a unique RKHS.
\par
In this paper, we focus on efficiently modeling nonlinear dynamics from nonstationary data streams and using the learned models for real-time forecasting.
Let $\stream = \{\obsone, ..., \obstc, ...\}$ be a semi-infinite multivariate data stream where $\tc$ is the current time and $\obst\in\R^{\d}$ is a $\d$-dimensional observation at time $\t$.
In our model, we assume that observation data $\obst$ behaves as the following time-varying discrete dynamical systems,
\begin{align}
  \label{eq:nlds}
  \obstt = \dynamics_\t(\obst)+\noise\ut,\quad\t\in\N,
\end{align}
where $\dynamics_\t$ is a nonlinear state-transition function describing the dynamics at time point $\t$,
and $\noise\ut$ denotes an i.i.d. zero-mean noise term.  
Nonstationarity mainly arises from the time-dependency of $\dynamics_\t$.
Thus, we tackle adaptively estimating nonlinear time-varying dynamical systems $\dynamics_\t$ from a data stream and forecast an $\ls$-steps-ahead future value, i.e., $\obs\tindex{\tc+\ls}$ sequentially.
Consequently, the problem addressed in our work is summarized as follows:
\begin{problem}
\textbf{Given}
a semi-infinite multivariate data stream $\stream$, which consists of $d$-dimensional vector $\obst$ generated sequentially, i.e., $\stream = \{ \obsone, ..., \obstc, ... \}$, where $\tc$ is the current time,
\begin{itemize}
    \item \textbf{Capture} complex nonlinear dynamic patterns over time,
    \item \textbf{Forecast} $\ls$-steps-ahead future values, i.e., 
    $\obs_{\tc+\ls}$,
\end{itemize}
continuously and quickly, in a streaming fashion.
\end{problem}

\subsection{Proposed Model -- \method}
We now explain our proposed model in detail.
To achieve our goal, we require \method model to be capable of the following.
\begin{itemize}[leftmargin=12.5pt]
    \item \textit{Dual-view kernelized system}:
    We formulate a probabilistic augmented state space model that treats both raw observations and RKHS features as emissions from a shared latent linear dynamical system.
    \item \textit{Switching structures}:
    Our model represents the underlying time-varying system $\dynamics_t$ as a switching dynamical system composed of multiple distinct local patterns.
\end{itemize}
\subsubsection{Dual-view kernelized system}
We begin with a simple case in which we have only a single dynamic pattern in a data stream.
To capture nonlinear dynamics within a linear state-space framework, we leverage the concept of the Koopman operator, which lifts the state into a higher-dimensional feature space where the dynamics become linear. Let $\psi: \mathbb{R}^d \to \mathbb{R}^m$ be a feature map associated with an RKHS, approximating the infinite-dimensional lifting.
We define an augmented observation vector $\mathbf{y}_t \in \mathbb{R}^{d+m}$ combining the dual views of the raw data and its feature representation
$$
\mathbf{y}_t=
\begin{bmatrix}
    \mathbf{x}_t\\
    \psi(\obs_t)
\end{bmatrix}.
$$
Ideally, $\psi(\mathbf{x}_t)$ is functionally dependent on $\mathbf{x}_t$, but modeling this deterministic constraint strictly within a generative probabilistic framework leads to a degenerate likelihood, making inference numerically unstable.
To address this, we adopt a relaxation strategy: we treat the augmented vector $\mathbf{y}_t$ as a noisy observation emitted from a low-dimensional latent vector $\latent_t \in \mathbb{R}^r$.
This formulation posits that the dominant dynamics evolve on a low-dimensional manifold~\cite{McQuarrie2021-ej}.
Furthermore, this relaxation allows the probabilistic model to absorb approximation errors arising from the finite-dimensional truncation of the Koopman operator into the observation noise.
Consequently, we describe the single dynamic pattern using the following equation:
\begin{model}
    \label{model:single}
    Let
    $\latent\ut\in\R^\rank$ be the latent vector at time $\t$,
    and $\vy\ut\in\R^{d+m}$ the augmented observation vectors at time $\t$.
    The following equation governs the single nonlinear dynamic pattern.
    \begin{align}
        \label{eq:single}
        \begin{split}
            \latent\utt&=\mA\latent\ut+\boldsymbol\xi\ut,\quad
            \boldsymbol\xi\ut\sim\mathcal{N}(\vect{0},\mat{Q}) \\
            \mathbf{y}_t&=
            \mH
            \latent\ut+\boldsymbol\eta\ut,\quad
            \boldsymbol\eta\ut\sim\mathcal{N}(\vect{0},\mat{R}) \\
        \end{split}
    \end{align}
    with initial conditions $\latent_1\sim\mathcal{N}(\initcond_0,\initcov_0)$
    In this equation,
    $\mA\in\R^{\rank\times\rank}$ is a transition matrix and
    $\,\mH=[\mC^\top\ \mW^\top]^\top\in\R^{(d+m)\times\rank}$ is an augmented observation matrix, where $\mC\in\R^{d\times\rank}$ maps the latent state to the raw observation and $\mW\in\R^{m\times\rank}$ maps it to the feature space.
    The process noise covariance is $\mQ\in\R^{\rank\times\rank}$ and
    the observation noise covariance is given by $\mathbf{R} = \operatorname{BlockDiag}(\mathbf{R}_{\mathbf{x}}, \mathbf{R}_{\psi}) \in \mathbb{R}^{(d+m)\times (d+m)}$, where $\mathbf{R}_{\mathbf{x}}\in\R^{d\times d}$ represents measurement noise and $\mathbf{R}_{\psi}\in\R^{m\times m}$ accommodates the kernel approximation residuals.
\end{model}
Here, we enforce this block-diagonal structure by assuming statistical independence between the measurement noise in $\obs_t$ and the approximation residuals in $\psi(\obs_t)$. This structural constraint decouples the two error sources, thereby stabilizing the inference and preventing the model from overfitting to the deterministic dependency between the views.
In addition, while \cref{model:single} is formulated as a linear system,
it can accurately capture a wide range of complex nonlinear dynamics, which is supported by Koopman operator theory.
Consequently, we have the following:
\begin{definition}[Dual-view kernelized system: DKS]
    Let $\regime$ be parameters of a dual-view kernelized system for a single nonlinear dynamic pattern,
    i.e., $\regime=\{\mA,\mQ,\mC,\mW,\mR_x,\mR_\psi,\initcond_0,\initcov_0\}$.
\end{definition}
\subsubsection{Switching structures}
\cref{model:single} is suitable for summarizing time series data generated by a time-invariant system.
However, considering the real-world scenario where dynamic patterns are ever-changing based on $\dynamics_t$, its ability is still insufficient for data stream processing.
To overcome this limitation, we introduce a more comprehensive model.
We assume that time-varying systems are approximated by discrete switching dynamical systems.
This assumption is generally valid because the dynamic patterns do not frequently change over time.
Then, the data stream $\stream$ is summarized by the model set $\{\regime_1, \ldots, \regime_R\}$,
where $R$ is the number of models.
Consequently, we have the following:
\begin{definition}[Dual-view kernelized system set]
    Let $\regimeset$ be a parameter set of multiple dual-view kernelized systems,
    i.e., $\regimeset=\{\regime_1, ..., \regime_\nrgm\}$,
    which describes distinct nonlinear dynamic patterns in a data stream.
\end{definition}

\section{Optimization Algorithm}
    \label{section:algorithm}
    \noindent In this section, we present an efficient algorithm with which we estimate the model set $\regimeset$ describing a semi-infinite data stream $\stream$ continuously.
We first introduce a static optimization to estimate model parameters from time series data, and then present a streaming algorithm for monitoring current patterns and forecasting in a streaming fashion based on the static optimization.

\begin{figure*}[t]
    \begin{center}
        \includegraphics[width=1.0\linewidth]{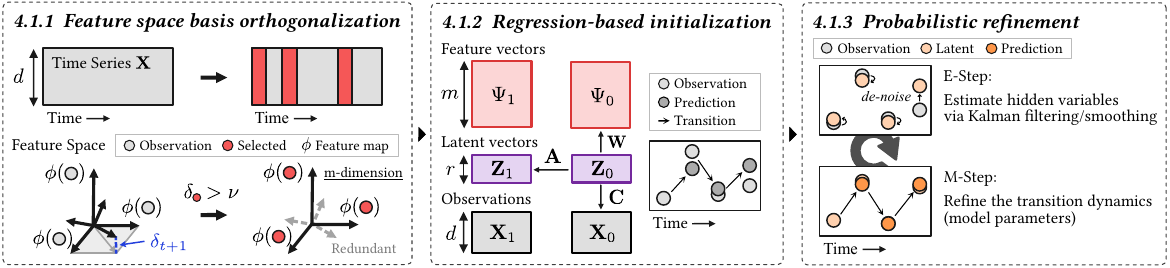}
        \caption{Overview of the static optimization framework.
        \method estimates the model parameters $\regime$ from a time series $\stream$ through three stages:
        (1) \textit{Feature space basis orthogonalization} constructs a sparse dictionary by selecting observations that exceed a linear dependence threshold $\nu$;
        (2) \textit{Regression-based initialization} performs spectral decomposition on the feature matrices $\Psi$ to initialize the latent linear dynamics $(\tilde{\mA},\tilde{\mC},\tilde{\mW})$; and
        (3) \textit{Probabilistic refinement} iteratively optimizes the latent states via the expectation-maximization (EM) algorithm, leading to dynamical consistency and robustness to noise.}
        \label{fig:static_optimization}
    \end{center}
    \vspace{-1em}
\end{figure*}

\subsection{Static Optimization}
\label{section:static_optimization}
Here, we describe the static optimization to estimate the best model parameter $\regime$ for a time series ${\stream}\in\R^{\d\times\n}$.
\cref{fig:static_optimization} illustrates the overall framework and \cref{alg:learning} (see \cref{appendix:alg}) shows its estimation procedure.
The static optimization consists of the following three components:
\begin{itemize}[leftmargin=12.5pt]
\item \textit{Feature space basis orthogonalization}: We construct a sparse finite-dimensional feature space from the RKHS to ensure computational stability and efficiency.
\item \textit{Regression-based initialization}: Spectral initialization of the model parameters via reduced rank regression (RRR), which interprets the system dynamics as a finite-rank approximation of the Koopman operator.
\item \textit{Probabilistic refinement}: We refine the initialized parameters by maximizing the marginal likelihood using the EM algorithm, treating the latent states as hidden variables.
\end{itemize}
\subsubsection{Feature space basis orthogonalization}
In this section, we describe how to efficiently optimize our proposed model.
Generally, kernel-based approaches tend to be computationally expensive as the dataset size increases.
This property is the major drawback for data stream processing~\cite{engel2004kernel}.
Additionally, linearly dependent feature vectors in the RKHS may lead to redundant representations and numerical instability.
Hence, we use the criterion of whether the model should incorporate an observed data point.
According to the representer theorem~\cite{kimeldorf1971some},
the optimal solution is given by a linear combination of $\{k(\cdot, \obs\ut)\}^{\n}_{\t=1}$, i.e.,  
$\dynamics(\obs)=\sum_{\t=1}^{\n}\beta_\t k(\obs, \obs\ut)$.
This implies that if $\fmap(\obs\ut)$ lies approximately in the span of $\{\fmap(\obs_\tau)\}_{\tau<\t}$,
then adding $\obs\ut$ does not enlarge the hypothesis subspace and the corresponding coefficient $\beta_t$ can be set to zero.
Motivated by this observation,
we maintain a dictionary of selected observations $\dict\ut=\{\obs_\tau\mid\tau\in\indices\!\ut\}$,
where $\indices\!\ut\subseteq[\t]$ represents the indices of selected data points, i.e., $|\indices_\n|=m_\n$.
We define the subspace spanned by the dictionary $\hat{\Hhil}_t=\mathtt{Span}\{\fmap(\obs_i):i\in\indices_t\}\subseteq\Hhil$ and $\Pi_{\dict\ut}$ denotes the orthogonal projection onto $\hat{\Hhil}_t$.
Specifically,
for every $\t\in[\n-1]$,
the above allocation for an observation $\obs\utt$ is realized using the following relation:
\begin{align}
    \label{eq:ald}
    \delta\utt&=\Hnorm{\fmap(\vx\utt)-\Pi_{\dict\ut}\fmap(\vx\utt)}^2,
\end{align}
which admits the kernel form
\begin{align}
    \label{eq:kernel_ald}
    \delta\utt&=k(\vx\utt,\vx\utt)-\vk_t^\top(\vx\utt)\mK_t^{-1}\vk\ut(\vx\utt),
\end{align}
where $\mK_t$ is the Gram matrix of the dictionary $\dict\ut$ and
$\vk\ut(\vx\utt)=[k(\vx_i,\vx\utt)]^\top_{i\in\indices\!\ut}$ is the vector between $\vx\utt$ and $\dict\ut$.
If $\delta\utt>\nu$ holds,
a new data point $\obs\utt$ is added to the dictionary $\dict\ut$.
In this case, the inverse Gram matrix $\mK^{-1}_{t}$ can be updated recursively via the block-matrix inverse:
\begin{align}
    \label{eq:ald_update}
    \mK_{t+1}^{-1}=
    \begin{bmatrix}
        \mK_{t}^{-1}+\delta_{t+1}^{-1}\vc\ut\vc_t^\top & -\delta_{t+1}^{-1}\vc_t\\
        -\delta_{t+1}^{-1}\vc_t^\top & \delta_{t+1}^{-1}
    \end{bmatrix},
\end{align}
where $\vc\ut=\mK^{-1}_t\vk_t(\vx\utt)$.
This formula is derived from the Woodbury identity and allows us to avoid recomputing a matrix inverse from scratch at every step.
The dictionary $\dict_\n$ provides a finite-dimensional feature vector $\psi(\obs)=\left[k(\obs_i,\obs)\right]_{i\in\indices_\n}^\top\in\R^{m}$,
where we denote $m\coloneqq m_\n$ as the dimension in the feature space for brevity.
Importantly,
$\psi(\obs)$ is the vector of inner products between $\phi(\obs)$ and the dictionary $\dict_\n$, i.e., $\psi(\obs)=[\Hinner{\phi(\obs_i)}{\phi(\obs)}]_{i\in\indices_\n}$.
Hence, $\psi$ provides a finite-dimensional representation of $\fmap$ restricted to the subspace $\hat{\Hhil}_\n$.
Note that this feature vector $\psi(\obs)$ corresponds to one in \cref{model:single}.
We conclude this section by introducing the feature matrices as follows:
\begin{align}
\label{eq:feature_mat}
\Psi=[\psi(\obs\ut)]_{\t=1}^\n,~~
\Psi_0=[\psi(\obs\ut)]_{\t=1}^{\n-1},~~
\Psi_1=[\psi(\obs\ut)]_{\t=2}^\n.
\end{align}

\subsubsection{Regression-based initialization}
Here, we explain the coarse parameter estimation based on the Koopman operator regression from the feature matrices.
First, we focus on a transition $\mF\in\R^{m\times m}$ in the feature space as follows:
\begin{align}
\label{eq:problem_F}
\min_{\mF}\frac{1}{\n-1}\|\Psi_1-\mF\Psi_0\|_\mathrm{F}^2+\lambda_A\|\mF\|_\mathrm{F}^2,
\end{align}
where $\lambda_A$ is a regularization parameter.
Reduced rank regression (RRR)~\cite{izenman1975reduced} interprets the system dynamics as a low-rank linear mapping in the feature space,
yielding a spectral initialization for the latent linear system.
To solve the problem, we first compute the empirical moment matrices
\begin{align}
    \label{eq:moment_mat}
    \mS_{00}=\frac{1}{\n-1}\Psi_0\Psi_0^\top+\lambda_A\idmat,\quad
    \mS_{10}=\frac{1}{\n-1}\Psi_1\Psi_0^\top.
\end{align}
We then define the whitened cross-covariance $\mM=\mS_{10}\mS_{00}^{-1/2}$, and compute its truncated singular value decomposition $\mM\approx\mU_\rank\boldsymbol\Sigma_\rank\mV_\rank^\top$,
where $\mU_\rank/\mV_\rank$ denotes the leading $\rank$ left/right singular vectors and $\boldsymbol\Sigma_\rank=\mathrm{diag}(\sigma_1,\ldots,\sigma_\rank)$ is the corresponding singular values.
The dimension of the subspace $\rank$ is automatically determined by~\cite{Gavish2014-dg}.
This procedure identifies the $\rank$-dimensional subspace in feature space that best explains the one-step temporal cross-covariance.
Specifically, we set the projection matrix $\tilde{\mW}=\mU_\rank$ and the latent transition matrix $\tilde{\mA}=\boldsymbol\Sigma_\rank\mV_\rank^\top\mS_{00}^{-1/2}\mU_\rank$. 
Here, $\tilde{\mW}$ spans the dominant predictive feature subspace.
Equivalently, when $\tilde{\mW}$ has orthonormal columns,
$\tilde{\mA}$ can be read as the projected dynamics $\tilde{\mA}=\tilde{\mW}^\dagger\mF\tilde{\mW}$, where $\dagger$ denotes the Moore--Penrose pseudoinverse.

Next, we estimate the observation matrix $\tilde{\mC}$ that maps the latent states back to the observation space.
Specifically, we first compute the projected latent states as $\mZ=\tilde{\mW}^\dagger\Psi$. Then, $\tilde{\mC}$ is obtained by a simple linear regression problem, i.e., $\min_{\mC}||\mX - \mC\mZ||_{\mathrm{F}}^2$, yielding the closed-form estimator $\tilde{\mC}=\mX\mZ^\top(\mZ\mZ^\top)^{-1}$.
We omit regularization, such as kernel ridge regression, because the low-dimensional latent space already serves as a regularizer.
This concludes the spectral initialization of the latent linear system $\{\tilde{\mA},\tilde{\mC},\tilde{\mW}\}$, helping to avoid poor local optima and accelerate convergence.

\subsubsection{Probabilistic refinement}
Since the regression-based initialization is fitted to one-step feature transitions, the resulting coarse estimators $\{\tilde{\mA},\tilde{\mC},\tilde{\mW}\}$ can be sensitive to noise and may be suboptimal in terms of the dynamical consistency of the latent trajectory.
To mitigate these issues, we employ the expectation-maximization (EM) algorithm~\cite{em} for the refinement of the coarse model parameters and latent vectors.
Specifically, we initialize the EM algorithm with the coarse estimators $\{\tilde{\mA},\tilde{\mC},\tilde{\mW}\}$.
Thanks to the linear representation based on Koopman operator theory, we can avoid computationally expensive sampling-based inference (e.g., particle filtering), commonly required by conventional nonlinear state space models.
First, we aim to infer the posterior distributions of the
latent vectors $\latent\ut$.
We can efficiently solve these inference
problems using the classical sum-product message passing algorithms~\cite{pearl1982reverend} in probabilistic graphical models.
In the context of the linear dynamical system, this inference procedure corresponds to applying the Kalman filter in the forward direction and the Rauch–Tung–Striebel (RTS) smoother in the backward pass.
Specifically, the forward passing equations are described by
\begin{align}
    \label{eq:pred_mu}
    \boldsymbol{\mu}_{t|t-1}&=\mA\boldsymbol{\mu}_{t-1|t-1}\\
    \label{eq:pred_P}
    \initcov_{t|t-1}&=\mA\initcov_{t-1|t-1}\mA^\top+\mQ\\
    \label{eq:kalman_gain}
    \mG_t&=\initcov_{t|t-1}\mH^\top(\mH\initcov_{t|t-1}\mH^\top+\mR)^{-1}\\
    \label{eq:filter_mu}
    \boldsymbol{\mu}_{t|t}&=\boldsymbol{\mu}_{t|t-1}+\mG_t(\vy\ut-\mH\boldsymbol{\mu}_{t|t-1})\\
    \label{eq:filter_P}
    \initcov_{t|t}&=(\idmat-\mG_t\mH)\initcov_{t|t-1},
\end{align}
where $\mG_t$ is the Kalman gain, $p(\latent\ut\mid\vy_{1:\t})=\mathcal{N}(\boldsymbol{\mu}_{t|t},\initcov_{t|t})$ with $\initcond_{1|0}=\initcond_{0}$ and $\initcov_{1|0}=\initcov_0$.
Also, the backward passing equations are
\begin{align}
    \label{eq:smoother_gain}\mJ_t&=\initcov_{t|t}\mA^\top\initcov_{t+1|t}^{-1}\\
    \label{eq:smoother_mu}\boldsymbol{\mu}_{t|\n}&=\boldsymbol{\mu}_{t|t}+\mJ_t(\boldsymbol{\mu}_{t+1|\n}-\boldsymbol{\mu}_{t+1|t})\\
    \label{eq:smoother_P}\initcov_{t|\n}&=\initcov_{t|t}+\mJ_t(\initcov_{t+1|\n}-\initcov_{t+1|t})\mJ_t^\top,
\end{align}
where $\mJ_t$ is the smoother gain and $p(\latent\ut\mid\vy_{1:\n})=\mathcal{N}(\boldsymbol{\mu}_{t|\n},\initcov_{t|\n})$.
Based on the smoothed latent vectors $\{\latent_t\}$,
we update the parameters in closed form:
\begin{align}
    \label{eq:update_initial}
    \initcond_0^{new}&=\initcond_{1|\n},\quad
    \boldsymbol\initcov_0^{new}=\initcov_{1|\n}, \\
    \label{eq:update_A}\mA^{new}&=\left(\sum_{\t=1}^{\n-1}\Exp{\latent\utt\latent_t^\top}\right)
    \left(\sum_{\t=1}^{\n-1}\Exp{\latent\ut\latent_t^\top}+(\n-1)\lambda_A\idmat\right)^{-1},\\
    \label{eq:update_H}\mH^{new}&=\left(\sum_{\t=1}^{\n}\vy_t\Exp{\latent_t^\top}\right)\left(\sum_{\t=1}^{\n}\Exp{\latent\ut\latent_t^\top}\right)^{-1},
\end{align}
and we split $\mH$ into $\mC$ and $\mW$ accordingly.
The noise covariances are updated by the residual second moments:
\begin{align}
    \label{eq:update_Q}\mQ^{new}&=\frac{1}{\n-1}\sum_{\t=1}^{\n-1}
    \Exp{\left(\latent\utt-\mA^{new}\latent\ut\right)
    \left(\latent\utt-\mA^{new}\latent\ut\right)^\top},\\
    \label{eq:update_Rx}\mR_x^{new}&=\frac{1}{\n}\sum_{\t=1}^{\n}
    \Exp{\left(\vx\ut-\mC^{new}\latent\ut\right)
    \left(\vx\ut-\mC^{new}\latent\ut\right)^\top},\\
    \label{eq:update_Rpsi}\mR_\psi^{new}&=\frac{1}{\n}\sum_{\t=1}^{\n}
    \Exp{\left(\psi(\obs\ut)-\mW^{new}\latent\ut\right)\left(\psi(\obs\ut)-\mW^{new}\latent\ut\right)^\top},
\end{align}
where $\Exp{\latent_t}=\initcond_{t|\n}$, $\Exp{\latent_{t+1}\latent_t^\top}=\initcov_{t+1|\n}\mJ_t^\top+\boldsymbol{\mu}_{t+1|\n}\boldsymbol{\mu}_{t|\n}^\top$,
and $\Exp{\latent_{t}\latent_t^\top}=\initcov_{t|\n}+\boldsymbol{\mu}_{t|\n}\boldsymbol{\mu}_{t|\n}^\top$.
Lastly, we introduce the key concept required for the subsequent section and present the theoretical evaluation of time complexity.
\begin{definition}[Sufficient statistics set]
\label{def:suff}
Let $\cS=\{\mS_1,\mS_2,\mS_3\}$ be the expected sufficient statistics computed with respect to the smoothed posterior distributions,
defined as
{\small
\begin{align*}
    \mS_1=\frac{1}{\n}\sum_{t=1}^\n\Exp{\latent_{t}\latent_t^\top},\,\,\,
    \mS_2=\frac{1}{\n-1}\sum_{t=1}^{\n-1}\Exp{\latent_{t+1}\latent_t^\top},\,\,\,
    \mS_3=\frac{1}{\n}\sum_{t=1}^\n\vy_t\Exp{\latent_t^\top},
\end{align*}
}
which are required for the adaptive parameter update in \cref{section:adaptive_param_update}.
\end{definition}

\begin{lemma}[Time complexity of static optimization]
\label{theorem:time_complexity_of_static_optimization}
Let $m$ be the final dictionary size produced by the basis orthogonalization, and let $\rank$ be the latent dimension.
Then, the time complexity of the static optimization is $O(\n m^2+m^3+\textnormal{\#}\:\!iter\cdot\n\rank^2(d+m))$.
See \cref{proof:time_complexity_of_static_optimization} for details.
\end{lemma}

\begin{figure}[t]
    \begin{center}
        \includegraphics[width=0.95\linewidth]{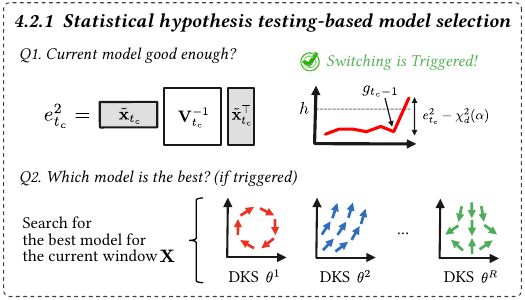}
        \caption{Statistical hypothesis testing-based model selection in \method:
        Given new data $\vx_\tc$,
        the normalized innovation squared (NIS) is monitored by a one-sided CUSUM test.
        If the test triggers switching, it evaluates stored models on the current window $\stream_c$.
        When no model handle the current pattern,
        it estimates new model and adds it to $\regimeset$.}
        \label{fig:model_select}
    \end{center}
\end{figure}

\subsection{Streaming Algorithm}
We describe how the \method model is utilized for data stream processing.
\cref{alg:model} (see \cref{appendix:alg}) describes our streaming algorithm in detail.
Given a new data $\vx_\tc$,
our streaming algorithm sequentially applies the following three steps:
\begin{itemize}[leftmargin=12.5pt]
    \item \textit{Statistical hypothesis testing-based model selection}:
    We employ a change detection mechanism based on the cumulative sum (CUSUM) test~\cite{cusum} of innovation statistics.
    This step determines the best model for the current window $\stream_c$.
    \item \textit{Adaptive parameter update}: Once a model is selected, we recursively estimate its parameters via an online EM algorithm using sufficient statistics.
    \item \textit{Forecasting via latent dynamics}:
    Utilizing the refined parameters and the current latent state estimated by the Kalman filter, we project the system dynamics forward to generate multi-step predictions.
\end{itemize}
Here, we define the parameter set required for these procedures.
\begin{definition}[Model candidate]
    Let $\candparam$ be a parameter set for the current window $\stream_c=\stream[\ts:\tc]$ spanning from the time point $\ts$ to $\tc$, i.e., $\candparam=\{\regime_c, \cS_c, \dict_c,\initcond_c,\initcov_c\}$, where $\regime_c$ best summarizes the current window $\stream_c$ and $\{\initcond_c,\initcov_c\}$ represents the filtered estimators of the latent state at the current time $\tc$,
    which are needed to generate an $\ls$-steps-ahead future value $\est_{\tc+\ls}$.
\end{definition}
\subsubsection{Statistical hypothesis testing-based model selection}
We incrementally update the full parameter set $\regimeset$ and the model candidate $\candparam$ given a new observation $\obstc$ and the current window $\stream_c=\stream[\ts:\tc]$.
We monitor the normalized innovation squared (NIS) derived from the Kalman filter innovations, which follows a $\chi^2$ law under the linear Gaussian assumption.
The one-step-ahead prediction is
\begin{align}
    \label{eq:pred_state}
    \hat{\initcond}^{new}_c
    =\mA\initcond^{prev}_c,\quad
    \hat{\initcov}^{new}_c=\mA\initcov^{prev}_c\mA^\top+\mQ.
\end{align}
The corresponding innovation for the observed channel is
\begin{align}
    \label{eq:innovation}
    \tilde{\obs}_{\tc}=\obstc-\mC\hat{\initcond}^{new}_c,\quad
    \mV_c=\mC\hat{\initcov}^{new}_c\mC^\top+\mR_x,
\end{align}
and we define the normalized innovation squared (NIS) by
\begin{align}
    \label{eq:nis}
    e_{\tc}^2=\tilde{\obs}_{\tc}^\top \mV_c^{-1}\tilde{\obs}_{\tc}.
\end{align}
Under the nominal model, $e_{\tc}^2$ approximately follows a $\chi^2$ distribution with $\d$ degrees of freedom.
Thus, given a confidence level $\alpha\in(0,1)$, we set a threshold $\chi^2_{\d}(1-\alpha)$ and apply a one-sided CUSUM test:
\begin{align}
    \label{eq:cusum}
    g_{\tc}=\max\left\{0,\,g_{\tc-1}+\bigl(e_{\tc}^2-\chi^2_{\d}(1-\alpha)\bigr)\right\}.
\end{align}
If $g_{\tc}$ exceeds a switching limit $h$, we consider the current model to be inconsistent and search for a better model.

When switching is triggered, it validates each model $\regime\in\regimeset$ using the current window $\stream_c$.
For every candidate $\regime\in\regimeset$, we reset its latent state to a broad prior; then refine the initial latent moments via a smoothing-based initialization.
After this initialization, we run a Kalman filter over the window and compute the NIS sequence $\{e_t^2\}_{t=\ts}^{\tc}$ based on \cref{eq:nis}.
We select the model that minimizes the mean value of the NIS sequence.
However, if there is no model such that the maximum value of the NIS sequence is lower than $h$, \method creates a new model from $\stream_c$ via the static optimization and inserts it into $\regimeset$ to handle an unknown pattern adaptively.

\begin{table*}[t]
\centering
\normalsize
\caption{
Multivariate forecasting results averaged across $71$ datasets, where the best results are in \textbf{bold} and the second best are underlined.
\method consistently outperforms its baselines.
Full results for all datasets are provided in \cref{appendix:results}.
}
\vspace{-0.5em}
\begin{tabular}{c|ccc|ccc}
\toprule
Metrics &
\multicolumn{3}{c|}{\textit{Mean Squared Error (MSE)}} &
\multicolumn{3}{c}{\textit{Mean Absolute Error (MAE)}} \\
\midrule

Steps &
$20$ & $25$ & $30$ &
$20$ & $25$ & $30$ \\
\midrule 

sKAF (2023) &
$34.0 \pm 112$ & $59.6 \pm 285$ & $81.4 \pm 270$ &
$1.20 \pm 1.96$ & $1.41 \pm 2.35$ & $1.72 \pm 2.61$ \\
Koopa (2023) &
$0.397 \pm 0.417$ &$0.377 \pm 0.331$  & $0.361 \pm 0.338$ & 
$0.430 \pm 0.165$ & $0.427 \pm 0.154$ & $0.423 \pm 0.144$ \\
OneNet (2023) &
$0.318 \pm 0.148$ & $0.321 \pm 0.146$ & $\underline{0.322 \pm 0.145}$ & 
$0.441 \pm 0.123$ & $0.443 \pm 0.122$ & $0.444 \pm 0.121$ \\
PAttn (2024) &
$0.311 \pm 0.206$ & $0.326 \pm 0.206$ & $0.330 \pm 0.196$ & 
$0.407 \pm 0.139$ & $0.417 \pm 0.138$ & $0.421 \pm 0.134$ \\
WPMixer (2025) &
$0.561 \pm 1.95$ & $0.509 \pm 1.56$ & $0.436 \pm 0.916$ &
$0.411 \pm 0.233$ & $\underline{0.413 \pm 0.214}$ & $\underline{0.413 \pm 0.194}$ \\
ModePlait (2025) &
$\underline{0.279 \pm 0.0918}$ & $\underline{0.300 \pm 0.104}$ & $0.325 \pm 0.115$ & 
$\underline{0.407 \pm 0.0857}$ & $0.423 \pm 0.0914$ & $0.441 \pm 0.0964$ \\

\midrule
\method (Ours) &
$\mathbf{0.0763\pm 0.0539}$ & $\mathbf{0.0999\pm 0.0659}$ & $\mathbf{0.120\pm 0.0747}$ & 
$\mathbf{0.183\pm 0.0769}$ & $\mathbf{0.214\pm 0.0857}$ & $\mathbf{0.240\pm 0.0888}$ \\
\bottomrule
\end{tabular}
\label{table:detailed_results}
\end{table*}

\subsubsection{Adaptive parameter update}
\label{section:adaptive_param_update}
Once a model $\regime_c$ is selected, it performs a one-step online EM update with a forgetting factor $\gamma\in(0,1]$.
This update consists of (i) feature-space adaptation by \cref{eq:ald} and (ii) recursive updates of sufficient statistics.
Given a new point $\obstc$, we compute the residual $\delta_{\tc}$ as in \cref{eq:ald}.
First, if $\delta_{\tc}>\nu$ holds, we add $\obstc$ to $\dict_c$ and obtain the coefficient vector $\va_{\tc}=\mK^{-1}\vk(\obstc)$.
When the dictionary grows from $m$ to $m+1$, the augmented observation map $\mH$ must be expanded accordingly.
We expand $\mW$ by one row via $\va_{\tc}^\top \mW$ and extend the feature-noise covariance $\mR_\psi$ by a block-diagonal update.
Similarly, we expand the sufficient statistic $\mS_3$ so that its dimension matches $d+m$.
This step prevents reinitialization and enables stable updates even when the feature space is growing.
If the dictionary size exceeds $m_{\max}$ after an update, we employ a pruning strategy to remove the basis vector that is most linearly dependent on the others, thereby minimizing the loss of information in the feature space.
Specifically, it is a known property of the Schur complement that the reciprocal of the $j$-th diagonal element $1/[\mK^{-1}]_{jj}$ equals the squared residual error of approximating the $j$-th basis vector using the linear span of the remaining $m-1$ vectors.
Therefore, we select the index $j^* = \argmax_j [\mK^{-1}]_{jj}$ for removal.
We then update $\mK^{-1}$ by applying a rank-$1$ downdate operation, which is mathematically equivalent to reversing the recursive update in \cref{eq:ald_update}, and delete the corresponding rows and columns from $\mW$, $\mR_\psi$, and $\mS_3$.

Next, given $(\hat{\initcond}_c^{new},\hat{\initcov}_c^{new})$ in \cref{eq:pred_state},
we apply one Kalman update to obtain the filtered state $(\initcond_c^{new},\initcov_c^{new})$.
We introduce the sufficient statistics set rule.
Using the filtered moments $\E[\latent_{\tc}\latent_{\tc}^\top]=\initcov_c^{new}+\initcond_c^{new}(\initcond_c^{new})^\top$ and a filtered approximation of $\E[\latent_{\tc}\latent_{\tc-1}^\top]$, we update
\begin{align}
    \label{eq:online_stat1}
    \mS_1^{new}&\leftarrow(1-\gamma)\,\mS_1^{prev}+\gamma\,\E[\latent_{\tc}\latent_{\tc}^\top],\\
    \label{eq:online_stat2}
    \mS_2^{new}&\leftarrow(1-\gamma)\,\mS_2^{prev}+\gamma\,\E[\latent_{\tc}\latent_{\tc-1}^\top],\\
    \label{eq:online_stat3}
    \mS_3^{new}&\leftarrow(1-\gamma)\,\mS_3^{prev}+\gamma\,\vy_{\tc}\E[\latent_{\tc}^\top].
\end{align}
Then, the parameters are updated in closed form by
\begin{align}
    \label{eq:online_A}
    \mA^{new}&\leftarrow \mS_2^{new} (\mS_1^{prev} + \lambda_A\idmat)^{-1},\\
    \label{eq:online_H}
    \mH^{new}&\leftarrow \mS_3^{new} (\mS_1^{new})^{-1}.
\end{align}

\subsubsection{Forecasting via latent dynamics}
\noindent
Finally, it forecasts future values using the model candidate $\candparam$.
Let $\initcond_c^{new}$ be the filtered latent mean after processing $\obstc$.
Because the latent dynamics are linear, the $\ls$-steps-ahead future value is simply computed by
\begin{align}
    \label{eq:forecast_latent}
    \hat{\latent}_{\tc+\ls}=\mA^{\ls}\initcond_c^{new},\quad
    \hat{\obs}_{\tc+\ls}=\mC\hat{\latent}_{\tc+\ls}.
\end{align}
This forecasting step is computationally light and directly leverages the current model selected by the statistical tests.

\begin{lemma}[Time complexity of \method]
    \label{lemma:time}
    Let $m$ be the dictionary size, $R$ be the number of stored models, and $r$ be the latent dimension.
    The time complexity of \method is at least $O(m^2 + r^2(d+m))$ and at most $O(\n m^2+m^3+R\n\rank^2(d+m)+\textnormal{\#}\:\!iter\cdot\n\rank^2(d+m))$ per process.
    See \cref{proof:time} for details.
\end{lemma}
\noindent
\cref{lemma:time} indicates that the computational time of \method is constant with regard to the length of a data stream $\tc$.
Thus, it is practical for semi-infinite data streams in terms of execution speed. 

\section{Experiments}
    \label{section:experiments}
    We ran experiments to answer the following questions.
{\setlength{\leftmargini}{19.5pt}
\begin{enumerate}
    \renewcommand{\labelenumi}{\textit{Q\arabic{enumi}.}}
    \item 
    \textit{Accuracy}: How accurately does it forecast future values?
    \item 
    \textit{Scalability}: How does it scale in computational time?
    \item
    \textit{Nonlinear capability}:
    How well does it estimate and leverage complex nonlinear dynamics in a streaming fashion?
    \item
    \textit{Sensitivity analysis}:
    How do hyperparameters influence real-time forecasting performance?
\end{enumerate}
}
\subsection{Experimental Setup}
\subsubsection{Datasets}
We used the dysts benchmark~\cite{Dysts}, which offers a diverse and practical collection of 71 chaotic dynamical systems for evaluating time series forecasting and data-driven modeling techniques.
This benchmark is particularly noteworthy for its wide applicability across diverse fields, including astrophysics, climatology, and biochemistry.
Note that chaotic dynamical systems are a kind of nonlinear dynamical system, and some of them consist of multiple distinct dynamic patterns~\cite{farnoosh2021deep}.
Hence, this benchmark is suitable for our experimental evaluation.
\subsubsection{Baselines}
We compared \method with the following seven models for time series forecasting.
\begin{itemize}[leftmargin=12.5pt]
    \item 
    ModePlait~\cite{modeplait} is a streaming algorithm for the discovery of time-evolving causality and forecasting future values.
    \item 
    WPMixer~\cite{murad2025wpmixer} is an MLP-based method that uses multi-level wavelet decomposition to integrate information from both time and frequency domains.
    \item 
    PAttn~\cite{tan2024language} is a simple linear model with patching and attention for feature extraction.
    \item 
    OneNet~\cite{wen2023onenet} is an online ensembling network that handles concept drift by utilizing both channel independence and cross-variable dependency.
    \item 
    Koopa~\cite{liu2023koopa} disentangles non-stationary time series into time-variant and time-invariant components by approximating the Koopman operator.
    \item 
    Streaming KAF (sKAF)~\cite{giannakis2023learning} reduces computational costs by approximation using randomized features and Nystr\"om methods.
\end{itemize}

\subsection{Results}
\subsubsection{Q1. Accuracy}
\label{section:accuracy}
We first assessed the quality of \method in terms of $\ls$-steps-ahead forecasting accuracy.
We used two evaluation metrics, the mean square error (MSE) and the mean absolute error (MAE),
both of which provide better results when they are closer to zero.
\cref{table:detailed_results} shows the mean and standard deviation of the results across $5$ different seed values and $71$ datasets,
where the best and second-best levels of performance are shown in \textbf{bold} and \underline{underlined}, respectively.
We provide the critical difference diagrams of these results
and the full results for all datasets
in \cref{appendix:results}.
\method outperforms state-of-the-art baselines by employing the Koopman operator for the estimation of nonlinear dynamics and the adaptive processing of nonstationary data streams.
Although deep learning-based models exhibit high generality for time series modeling, they cannot adapt to time-varying environments and thus degrade significantly on datasets with distinct dynamic patterns.
ModePlait summarizes an entire data stream through the discovery of time-evolving causality.
However, since it is not specialized for modeling nonlinear dynamical systems,
its performance may degrade for highly complex phenomena.
Streaming KAF showed considerable instability in forecasting accuracy because its update rule cannot handle unknown dynamic patterns.
The results also showed that methods capable of handling time-varying dynamic patterns exhibit relatively small standard deviations.

\begin{figure}[t]
    \begin{center}
        \includegraphics[width=1.0\linewidth]{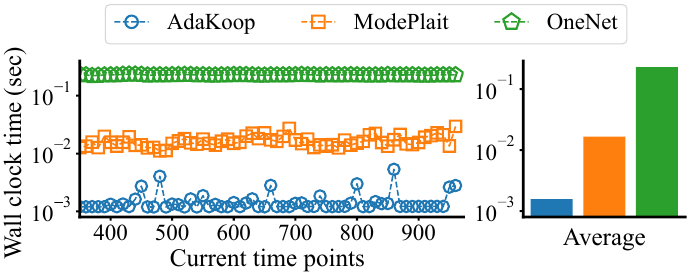}
        \caption{Scalability of \method: (left) Wall clock time vs. data stream length $\tc$ and (right) average time consumption.}
        \label{fig:time}
    \end{center}
    \vspace{-1.0em}
\end{figure}

\begin{table}[t]
\centering
\normalsize
\caption{
Forecasting results across different kernel functions.
}
\vspace{-0.5em}
\begin{tabular}{c|c|c|c}
\toprule
\multicolumn{2}{c|}{Metrics} &
\multicolumn{1}{c|}{MSE} &
\multicolumn{1}{c}{MAE} \\
\midrule


\multirow[t]{3}{*}{RBF}
 & 20 & 
$0.0763\pm 0.0539$ & $0.183\pm 0.0769$ \\
 & 25 & 
$0.0999\pm 0.0659$ & $0.214\pm 0.0857$ \\
 & 30 & 
$0.120\pm 0.0747$ & $0.240\pm 0.0888$ \\
\midrule
\multirow[t]{3}{*}{Polynomial}
& 20 &
$0.115 \pm 0.0922$ & $0.221 \pm 0.0953$ \\
& 25 &
$0.144 \pm 0.108$ & $0.254 \pm 0.102$ \\
& 30 &
$0.164 \pm 0.114$ & $0.278 \pm 0.103$ \\
\midrule
\multirow[t]{3}{*}{Sigmoid}
& 20 &
$0.218 \pm 0.0868$ & $0.365 \pm 0.0837$ \\
& 25 &
$0.239 \pm 0.0967$ & $0.384 \pm 0.0889$ \\
& 30 &
$0.253 \pm 0.101$ & $0.398 \pm 0.0920$ \\
\midrule
\multirow[t]{3}{*}{Linear}
& 20 &
$0.218 \pm 0.0870$ & $0.366 \pm 0.0846$ \\
& 25 &
$0.239 \pm 0.0963$ & $0.386 \pm 0.0893$ \\
& 30 &
$0.253 \pm 0.0999$ & $0.400 \pm 0.0917$ \\
\bottomrule
\end{tabular}
\label{table:kernel}
\vspace{-1.0em}
\end{table}

\subsubsection{Q2. Scalability}
We evaluated the performance of \method in terms of computational time.
\cref{fig:time} compares the computational efficiency of methods that handle time-varying nonlinear dynamic patterns in a data stream.
It presents the computational time at each time point $\tc$ on the left and the average time on the right.
Note that both figures are shown on linear-log scales.
\method consistently outperformed its competitors in terms of computational efficiency, thanks to the efficient streaming algorithm.
In addition, this result aligns with the discussion provided in \cref{lemma:time}.
OneNet requires significantly more time for model updates than other methods due to the use of two complementary models.
Although ModePlait operates at a moderate speed, it not only forecasts but also discovers causal relationships sequentially, which makes it slower than our algorithm.

\subsubsection{Q3. Nonlinear capability}
We demonstrated that \method accurately captures nonlinear dynamics and utilizes the learned models for real-time forecasting.
To validate this, we compared the RBF kernel, which meets our theoretical requirements, with other suboptimal kernels.
\cref{table:kernel} shows the forecasting results across various kernel functions.
The RBF kernel consistently yields the best forecasting performance, demonstrating its superior capability in capturing complex nonlinear dynamics as mentioned in \cref{section:accuracy}.
While the polynomial kernel achieves relatively high performance, it falls short of the RBF kernel because this kernel maps data into a finite-dimensional space.
In contrast, the linear kernel fails to capture complex nonlinear dynamics, leading to a substantial decrease in forecasting accuracy.
Although the sigmoid kernel is an infinite-dimensional kernel, it is not positive definite.
Consequently, it violates the theoretical requirements of the Koopman operator regression framework, resulting in performance comparable to that of a linear kernel.
These results empirically demonstrate that our probabilistic augmented state space model effectively leverages RKHS properties to estimate nonlinear dynamics despite being in linear system form and accurately performs real-time forecasting in a nonstationary environment.

\begin{figure}[t]
    \centering
    \includegraphics[width=1.0\linewidth]{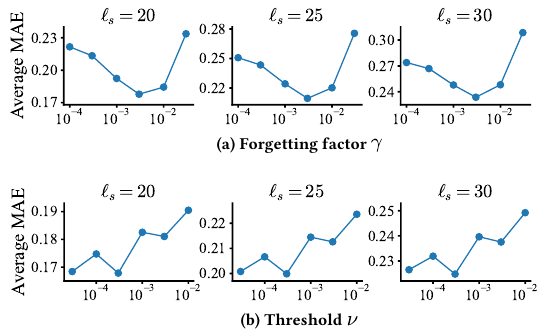}
    \vspace{-1.5em}
    \caption{Sensitivity analysis results.}
    \label{fig:hypara_sensitivity}
    \vspace{-0.5em}
\end{figure}

\subsubsection{Q4. Sensitivity analysis}
We evaluated the sensitivity of our algorithm to its hyperparameters.
\cref{fig:hypara_sensitivity} shows the forecasting accuracy (MAE) when varying hyperparameters.
\cref{fig:hypara_sensitivity}\,(a) shows the impact of the forgetting factor $\gamma$.
As $\gamma$ increases, \method incorporates more recent information and discards the past one.
This result indicates that the best $\gamma$ lies in $[10^{-3}, 10^{-2}]$.
When $\gamma$ is too small, the algorithm adapts too slowly to changes in a nonstationary data stream, resulting in poor tracking performance.
Conversely, a large $\gamma$ results in unstable estimation due to the insufficient effective sample size.
Next, \cref{fig:hypara_sensitivity}\,(b) illustrates the sensitivity to the threshold $\nu$.
This parameter controls the sparsity of the dictionary in the feature space.
We can observe that a tendency towards improved accuracy as $\nu$ decreases.
This is because that a smaller threshold enables the model to maintain a richer dictionary of basis vectors, allowing for a more precise approximation of complex nonlinear dynamics.
Note that small $\nu$ improves model expressiveness but increases computational costs, so we should select an appropriate value that balances accuracy and efficiency according to the requirements.

\section{Conclusion}
    \label{section:conclusion}
    In this paper, we presented an efficient streaming method, \method, designed for modeling nonlinear dynamics from nonstationary data streams with Koopman operator regression.
Our leading solution relies on a probabilistic dual-view state space model that treats both raw observations and RKHS features as emissions from a shared latent linear system, thereby enabling stable and efficient modeling of nonlinear dynamics.
To adaptively handle nonstationarity in data streams, our algorithm employs statistical hypothesis testing and continuous parameter updates, but its computational complexity remains constant regardless of the data stream length.
Experimental results across $71$ practical benchmark datasets showed that \method achieved remarkable improvements over competitors in both forecasting accuracy and computational efficiency.
The results provide empirical validation that \method is suitable for real-time forecasting in nonstationary data streams.

\section*{Acknowledgments}
We would like to thank the anonymous referees
for their valuable and helpful comments.
This work was partly supported by
``Program for Leading Graduate Schools'' of the Osaka University, Japan, 
JSPS KAKENHI Grant-in-Aid for Scientific Research Number
JP25KJ1729,    
JP26H02499,  
JST CREST JPMJCR23M3, 
JST K Program JPMJKP25Y6, 
JST COI-NEXT JPMJPF2009, 
JST COI-NEXT JPMJPF2115, 
Future Social Value Co-Creation Project - Osaka University. 

\bibliographystyle{bib/ACM-Reference-Format}
\bibliography{
bib/reference,
bib/general
}


\appendix
\section*{Appendix}
\label{appendix}
\section{Notation and Terminology}
\label{appendix:notation}
Table \ref{table:symbols} lists the symbols and definitions used in our work.
\begin{table}[ht]
    \centering
    \caption{Symbols and definitions.}
    \small
    \vspace{-0.5em}
    \resizebox{1.0\linewidth}{!}{
        \begin{tabular}{l|l}
            \toprule
            Symbol & Definition \\
            \midrule
            \multicolumn{2}{l}{\textit{General notation}}\\
            \midrule
            $\d$ & Number of dimensions \\
            $\n$ & Window size \\
            $\tc$ & Current time point \\
            $\vx\ut$ & Observation at time $\t$ \\
            $\stream$ & Nonstationary multivariate data stream (semi-infinite) \\
            $\stream_c$ & Current window, i.e., $\mX_c = \stream[t_s:\tc]\in\R^{\d\times\n}$ \\
            $\!\Hhil$ & Reproducing kernel Hilbert space (RKHS) \\
            $\fmap$ & Feature map, i.e., $\fmap:\M\to\Hhil$ \\
            $\>\!k$ & Kernel function, i.e., $k=\Hinner{\phi(\vx)}{\phi(\vy)}$ \\
            $\>\!\dagger$ & Moore--Penrose pseudoinverse \\
            \midrule
            \multicolumn{2}{l}{\textit{Dual-view kernelized system (DKS)}}\\
            \midrule
            $\>\!\!m$ & Dimension of feature space \\
            $\>\!\rank$ & Dimension of latent space \\
            $\vy_t$ & Augmented observations (two-views) \\
            $\latent_t$ & Latent state vector \\
            $\regime$ & DKS model, i.e., $\regime=\{\mA,\mQ,\mC,\mW,\mR_x,\mR_\psi,\boldsymbol\mu_0,\mP_0\}$\\
            $R$ & Number of models \\
            $\regimeset$ & DKS model set, i.e., $\regimeset 
             = \{ \regime_i \}_{i=1}^R$\\
            \midrule
            \multicolumn{2}{l}{\textit{Static optimization}}\\
            \midrule
            $\indices_t$ & Indices of selected data points \\
            $\dict_t$ & Dictionary, i.e., $\dict=\{\obs_\tau\mid\tau\in\indices_t\}$ \\
            $\hat{\Hhil}_t$ & Dictionary subspace, i.e., $\hat{\Hhil}_t=\mathtt{Span}\{\fmap(\obs_i):i\in\indices_t\}\subseteq\Hhil$ \\
            $\Pi_{\dict_t}$ & Orthogonal projection onto $\hat{\Hhil}$ \\
            $\mK\ut$ & Gram matrix over $\dict_t$ \\
            $\vk\ut(\vx)$ & Kernel vector between $\vx$ and $\dict_t$ \\
            $\delta_{t+1}$ & Residual distance of $\fmap(\vx_{t+1})$ to $\hat{\Hhil}_t$ \\
            $\psi(\vx)$ & Feature vector based on dictionary, i.e., $\psi(\obs)=\left[k(\obs_i,\obs)\right]_{i\in\indices_\n}^\top$ \\
            $\Psi$ & Feature matrix, i.e., $\Psi=[\psi(\obs\ut)]_{\t=1}^\n$ \\
            $\lambda_A$ & Regularization parameter \\
            $\!\!\{\tilde{\mA},\tilde{\mC},\tilde{\mW}\}$ & Coarse estimators for initializing the EM algorithm \\
            $\mG_t$ & Kalman gain \\
            $\,\mJ_t$ & Smoother gain \\
            $\textnormal{\#}\:\!iter$ & Number of EM iterations \\
            \midrule
            \multicolumn{2}{l}{\textit{Streaming algorithm}}\\
            \midrule
            $\>\!\!\cS$ & Sufficient statistical set, i.e., $\cS=\{\mS_1,\mS_2,\mS_3\}$ \\
            $\gamma$ & Forgetting factor \\
            $\ls$ & Forecasting steps \\
            ${\boldsymbol{\mu}}_c,{\mP}_c$ & Filtered latent mean/covariance at time $\tc$ \\
            $\cC$ & Model candidate, i.e., $\mathcal{C}=\{\theta_c,\cS_c,\dict_c,\boldsymbol\mu_c,\mP_c\}$ \\
            $\tilde{\vx}_{\tc}$ & Innovation, i.e., $\tilde{\vx}_{\tc}=\vx_\tc - \mC\hat{\boldsymbol{\mu}}^{new}_c$ \\
            $e_\tc^2$ & Normalized innovation squared (NIS) \\
            $\alpha$ & Confidence level \\
            $g_\tc$ & One-sided CUSUM statistic computed from the NIS sequence \\
            \bottomrule
        \end{tabular}
    }
    \normalsize
    \label{table:symbols}
    \vspace{-0.5em}
\end{table}

\section{Preliminaries}
\label{appendix:background}
\subsection{Koopman Operator Theory}
Let $\rand{X}\coloneqq\{\rand{X}_\t: \t\in\N\}$ be a family of random variables with values in a measurable space $(\M, \boldsymbol\Sigma_{\!\M})$, called state space. We assume that

\noindent
$\rand{X}$ is a time-homogeneous\footnote{Time-homogeneity denotes that $P(\randx_{t+1}\in B\mid\randx_t=x)$ for any measurable set $B$ does not depend on $t$.} Markov chain\footnote{Markov chain denotes that $\randx_{t+1}$ only depends on the present state $\randx_{t}$.} with a transition kernel $p: \M\times\Sigma_\M\to[0, 1]$ such that
\begin{align*}
    P(\randx_{t+1}\in B\mid\randx_t=x)=p(x,B),\quad\forall\t\in\N,~~\forall(x,B)\in\M\times\Sigma_{\!\M}.
\end{align*}
For a vector space $\F$ of $\Sigma_{\!\M}$-measurable observables $\eta:\M\to\C$,
the Markov transfer operator $\mathcal{K}_\F:\F\to\F$ is defined by
\begin{align*}
    (\mathcal{K}_\F\:\!\obssca)(x)\coloneqq\int_\M \obssca(y)p(x, dy)=\E[\obssca(\rand{X}_{\tt})\mid\rand{X}_\t=x],
\end{align*}
Some Markov chains admit the existence of an invariant probability measure $\invm$ that satisfies
$$
\invm(B)=\int_\M\invm(dx)p(x, B),\quad B\in\Sigma_\M.
$$
\begin{figure}[t]
\begin{algorithm}[H]
    \caption{\method($\mX_c, \regimeset, \candparam$)}
    \label{alg:model}
    \begin{algorithmic}[1]
        \Require (a) Current window $\mX_c\in\R^{\d\times\n}$ at current time $\tc$ 
        \Statex \hspace{4.2mm} (b) Full parameter set $\regimeset$,
        \Statex \hspace{4.3mm} (c) Model candidate $\candparam=\{\regime_c, \cS_c, \dict_c,\initcond_c,\initcov_c\}$
        \Ensure (a) Updated full parameter set $\regimeset'$
        \Statex \hspace{6.2mm} (b) Updated model candidate $\candparam'$
        \Statex \hspace{6.3mm} (c) $\ls$-steps-ahead future value $\est_{\tc+\ls}$
        \State $\hat{\boldsymbol\mu}_c=\mA\boldsymbol\mu_c$, $\hat{\mP}_c=\mA\mP_c\mA^\top+\mQ$ \Comment{\cref{eq:pred_state}}
        \State $\tilde{\obs}_\tc=\obs_\tc-\mC\hat{\boldsymbol\mu}_c$, $\mV_c=\mC\hat{\mP}_c\mC^\top+\mR_x$ \Comment{\cref{eq:innovation}}
        \State $e_{\tc}^2=\tilde{\obs}_{\tc}^\top \mV_c^{-1}\tilde{\obs}_{\tc}$ \Comment{\cref{eq:nis}}
        \State $g_{\tc}=\max\left\{0,\,g_{\tc-1}+\bigl(e_{\tc}^2-\chi^2_{\d}(1-\alpha)\bigr)\right\}$ \Comment{\cref{eq:cusum}}
        \If{$g_{\tc}>h$} 
        \State $\regimeset_h\leftarrow\{\regime\in\regimeset:\max(\{e^2_t\}_{t=\ts}^\tc)<h\}$
        \If{$|\regimeset_h|=0$}
            \State $\candparam'\leftarrow\textsc{StaticOptimization}(\stream_c)$ \Comment{\cref{alg:learning}}
            \State $\regimeset'\leftarrow\regimeset\cup\regime_c$
        \Else
            \State $\regime_c\leftarrow\argmin_{\regime\in\regimeset_h}\mathrm{mean}(\{e^2_t\}_{t=\ts}^\tc)$
        \EndIf
        \EndIf
        \If{NOT create new model}
            \State Update dictionary $\dict_c'$
            \Comment{\cref{eq:ald,eq:kernel_ald,eq:ald_update}}
            \State Update filtered estimators $\boldsymbol\mu_c',\mP_c'$ \Comment{\cref{eq:pred_mu,eq:pred_P,eq:kalman_gain,eq:filter_mu,eq:filter_P}}
            \State Update $\cS'_c$, $\mA'$, $\mH'$ \Comment{\cref{eq:online_stat1,eq:online_stat2,eq:online_stat3,eq:online_A,eq:online_H}}
            \State $\candparam'\leftarrow\{\regime_c,\dict_c',\cS'_c,\boldsymbol\mu_c',\mP_c'\}$
        \EndIf
        \State $\hat{\obs}_{\tc+\ls}=\mC\hat{\latent}_{\tc+\ls}$,\ \ $\hat{\latent}_{\tc+\ls}=\mA^{\ls}\initcond_c'$ \Comment{\cref{eq:forecast_latent}}
    \State\Return $\{\regimeset', \candparam', \est_{\tc+\ls}\}$
    \end{algorithmic}
\end{algorithm}
\vspace{-1.0em}
\end{figure}
\begin{figure}[ht]
\begin{algorithm}[H]
    \caption{\textsc{StaticOptimization}\,($\mX$)}
    \label{alg:learning}
    \begin{algorithmic}[1]
        \Require Time series data $\mX\in\R^{\d\times\n}$
        \Ensure Model Candidate $\candparam=\{\regime,\cS,\dict_\n,\boldsymbol{\mu}_{\n|\n},\mP_{\n|\n}\}$
        \State Initialize $\dict_{1}=\{\vx_1\}$; $\indices_{1}=\{1\}$
        \For{$\t\in[\n-1]$}
            \State $\delta\utt=k(\vx\utt,\vx\utt)-\vk_t^\top(\vx\utt)\mK_t^{-1}\vk\ut(\vx\utt)$ \Comment{\cref{eq:kernel_ald}}
            \If{$\delta\utt>\nu$}
                \State $\dict_{\t+1} \leftarrow \dict_{\t} \cup \{\obs\tindex{\t+1}\}$;
                $\indices_{\t+1} \leftarrow \indices_{\t} \cup \{\t+1\}$
                \State Update $\mK_{t}^{-1}$ to $\mK_{t+1}^{-1}$ \Comment{\cref{eq:ald_update}}
            \EndIf
        \EndFor
        \State Construct feature matrices $\Psi, \Psi_0, \Psi_1$ \Comment{\cref{eq:feature_mat}}
        \State Compute moment matrices $\mS_{00},\mS_{10}$ \Comment{\cref{eq:moment_mat}}
        \State $\{\mU_r,\boldsymbol\Sigma_\rank,\mV_\rank\}\leftarrow\textsc{TruncateSVD}(\mS_{10}\mS_{00}^{-1/2})$
        \State $\tilde{\mA}\leftarrow\boldsymbol\Sigma_\rank\mV_\rank^\top\mS_{00}^{-1/2}\mU_\rank$,\ \ $\tilde{\mW}\leftarrow\mU_r$ 
        \State $\tilde{\mC}\leftarrow\mX\mZ^\top(\mZ\mZ^\top)^{-1}$ \Comment{$\mZ=\tilde{\mW}^\dagger\Psi$}
        \Repeat
        \State Estimate $\{\boldsymbol{\mu}_{\t|\n}\}_{\t=1}^\n$, $\{\mP_{\t|\n}\}_{\t=1}^\n$ \Comment{\cref{eq:pred_mu,eq:pred_P,eq:kalman_gain,eq:filter_mu,eq:filter_P,eq:smoother_gain,eq:smoother_mu,eq:smoother_P}}
        \State Update $\regime=\{\mA,\mQ,\mC,\mW,\mR_x,\mR_\psi,\initcond_0,\initcov_0\}$ \Comment{\cref{eq:update_initial,eq:update_A,eq:update_H,eq:update_Q,eq:update_Rx,eq:update_Rpsi}}
        \Until{convergence}
        \State Set $\cS=\{\mS_1,\mS_2,\mS_3\}$ \Comment{\cref{def:suff}}
        \State\Return $\{\regime,\cS,\dict_\n,\boldsymbol{\mu}_{\n|\n},\mP_{\n|\n}\}$
    \end{algorithmic}
\end{algorithm}
\vspace{-1.0em}
\end{figure}

In this case,
it is natural to take $\F=\Ltwo$ and denote the corresponding Koopman operator by $\koop$.
With the Hilbert space inner product
$$\inner{\eta}{\xi}_{\Ltwo}\coloneqq\int_\M\eta(x)\overline{\xi(x)}\pi(dx),$$
the operator $\koop$ is bounded on $\Ltwo$ and, by Jensen's inequality together with the invariant measure $\pi$, is a contraction:
$$\norm{\koop\eta}_{\Ltwo}\leq\norm{\eta}_{\Ltwo}.$$
Intuitively, the norm $\norm{\eta}_{\Ltwo}$ measures the mean-square amplitude of the observable $\eta$ when states are sampled according to the invariant measure $\pi$.
The Koopman operator maps $\eta$ to the one-step-ahead conditional mean $(\koop\eta)(x)$.
Thus, in stochastic systems, $\koop$ acts as an averaging operator: it extracts the component of the future observable that is predictable from the current state.
Such averaging can remove unpredictable fluctuations, but it cannot increase the mean-square amplitude under the stationary weighting $\pi$.
Hence, $\koop$ is non-expansive in $\Ltwo$.
In addition, the classical Koopman operator for deterministic dynamical systems is recovered as a special case of this formulation.
Specifically, for a measurable map $T:\M\to\M$ satisfying $X_{t+1}=T(X_t)$,
setting $p(x,B)=\delta_{T(x)}(B)$ gives $(\koop\eta)(x)=\eta(T(x))$,
showing that the conditional-expectation formulation used here is a stochastic generalization of the classical deterministic Koopman framework.
\par
One of the main advantages of Koopman operator theory is that it represents the evolution of observables through a linear operator, enabling spectral analysis even when the underlying state evolution is nonlinear.
Indeed, in many situations, and notably for normal compact Koopman operators, the spectral theorem yields an orthonormal basis of a set of eigenfunctions $\{\eigf_1, \eigf_2, \ldots\}$ of $\koop$ with eigenvalues $\{\eigv_1, \eigv_2, \ldots\}$ such that $$\koop\eigf_i=\eigv_i\eigf_i,\quad\forall i\in\N,$$ where $\eigf_i:\M\to\C$ and $\eigv_i\in\C$.
Hence, every $\eta\in\Ltwo$ can be expressed as the superposition of simpler basis vectors as follows:
$$\eta=\sum_{i\in\N}\gamma_i\psi_i,\quad \gamma_i=\inner{\eta}{\psi_i}_{\Ltwo}.$$
Consequently, for $t\in\N$,
\begin{align*}
    (\koop^\t\obssca)&=\sum_{i\in\N}\eigv_i^t\mode_i\eigf_i,
\end{align*}
which is known as Koopman mode decomposition (KMD)~\cite{mezic2005spectral, budivsic2012applied}.
For a fixed observable $\eta$, the decomposition is described by a set of triples $\{ (\eigv_i, \eigf_i, \mode_i) \}_{i\in\N}$.
However, a Koopman operator is not compact or normal in general. Therefore, without the above additional assumptions, there is no general guarantee that its eigenfunctions form a complete orthonormal basis of the space, which makes learning KMD challenging.

\section{Algorithm}
\label{appendix:alg}
Here, we collect pseudocode for our proposed method.
\cref{alg:model} outlines the overall streaming framework, which executes adaptive model selection, parameter updates, and real-time forecasting.
\cref{alg:learning} details the static optimization procedure designed to estimate model parameters from a finite time series window.

\begin{figure*}[t]
    \centering
    \begin{minipage}{0.48\linewidth}
        \includegraphics[width=\linewidth]{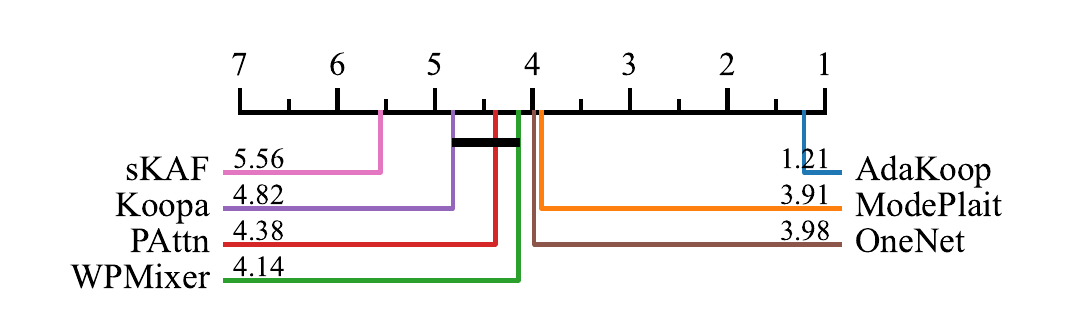}
        \subcaption{Mean Squared Error (MSE)}
    \end{minipage}
    \begin{minipage}{0.48\linewidth}
        \includegraphics[width=\linewidth]{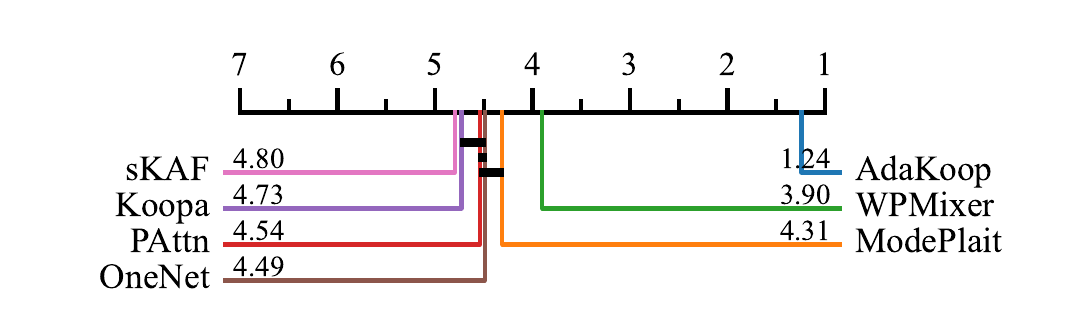}
        \subcaption{Mean Absolute Error (MAE)}
    \end{minipage}
    \caption{
        Critical difference diagrams of multivariate forecasting results. We used forecasting steps $\ls\in\{20, \ldots, 30\}$.
    }
    \label{fig:critical_difference}
    \vspace{-1.0em}
\end{figure*}


\begin{figure*}[t]
    \centering
    \begin{minipage}{0.48\linewidth}
        \includegraphics[width=\linewidth]{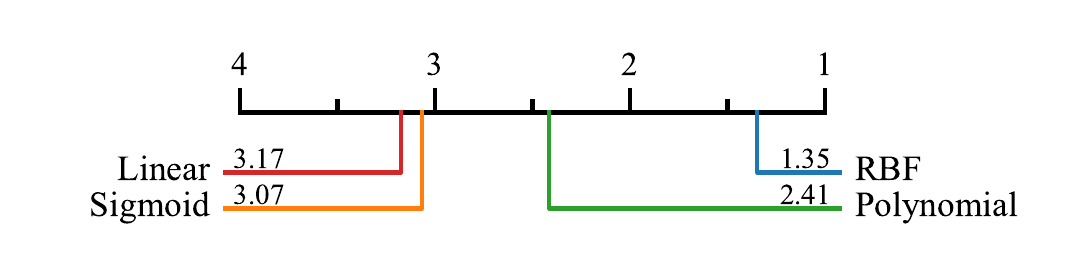}
        \subcaption{Mean Squared Error (MSE)}
    \end{minipage}
    \begin{minipage}{0.48\linewidth}
        \includegraphics[width=\linewidth]{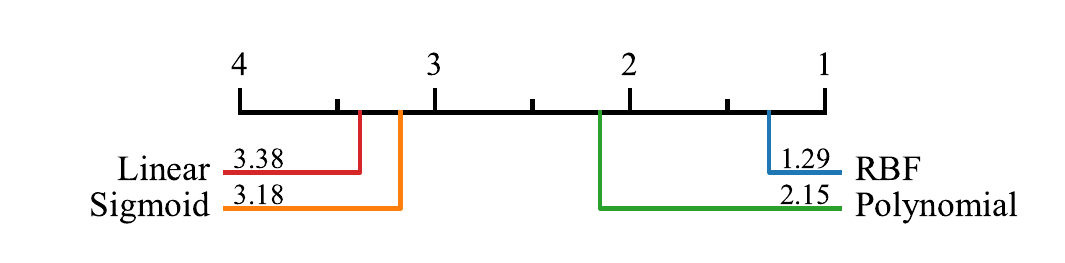}
        \subcaption{Mean Absolute Error (MAE)}
    \end{minipage}
    \caption{
        Critical difference diagrams of \method with various kernels.
        We used forecasting steps $\ls\in\{20, \ldots, 30\}$.
    }
    \label{fig:sensitivity_analysis}
\end{figure*}

\section{Proofs}
\subsection{Proof of \cref{theorem:time_complexity_of_static_optimization}}
\label{proof:time_complexity_of_static_optimization}

\begin{replemma}{theorem:time_complexity_of_static_optimization}[Time complexity of static optimization]
Let $m$ be the final dictionary size produced by the basis orthogonalization, and let $\rank$ be the latent dimension.
Then, the time complexity of the static optimization is $O(\n m^2+m^3+\textnormal{\#}\:\!iter\cdot\n\rank^2(d+m))$.
\end{replemma}

\begin{proof}
We derive the time complexity of the static optimization described in \cref{section:static_optimization}.
It consists of three stages; therefore, it suffices to sum the time complexities of each stage.

\mypara{(1) Feature space basis orthogonalization}
At each $t$, computing the residual
$\delta\utt$ in \cref{eq:ald}
requires forming the kernel vector $\vk_t(\vx_t)\in\R^{m_t}$ and a matrix--vector multiplication with $\mK_t^{-1}$, which costs $O(m_t^2)$.
If a point is added, updating $\mK_{t+1}^{-1}$ by \cref{eq:ald_update} is also $O(m_t^2)$.
Since $m_t\le m$ for all $t$, the total cost is $O(\n m^2)$.

\mypara{(2) Regression-based initialization}
Forming the moment matrices $\mS_{00}$ and $\mS_{10}$ in \cref{eq:moment_mat} costs $O(\n m^2)$.
Computing $\mS_{00}^{-1/2}$ via eigendecomposition costs $O(m^3)$.
Subsequent multiplications to form $\mM=\mS_{10}\mS_{00}^{-1/2}$ are $O(m^3)$ in the worst case, which is absorbed into $O(m^3)$.
Finally, computing the truncated SVD of an $m\times m$ matrix to obtain the leading $\rank$ components costs at most $O(m^2\rank)$, which is also dominated by $O(m^3)$ because $\rank\le m$.
Hence, the second step costs $O(m^3+\n m^2)$.

\mypara{(3) Probabilistic refinement}
Each EM iteration runs a Kalman filter forward pass and an RTS smoother backward pass for $\n$ steps.
With latent dimension $\rank$ and observation dimension $(d+m)$, if we compute the Kalman gain using the Woodbury identity, then the per-step cost is
$O(\rank^2(d+m)+\rank^3)=O(\rank^2(d+m))$ because $r<m$.
Therefore, the total cost per EM iteration is $O(\n\rank^2(d+m))$, and for $\textnormal{\#}\:\!iter$ iterations it is
$O(\textnormal{\#}\:\!iter\cdot\n\rank^2(d+m))$.
\end{proof}

\subsection{Proof of Lemma \ref{lemma:time}}
\label{proof:time}
\begin{replemma}{lemma:time}[Time complexity of \method]
    Let $m$ be the dictionary size, $R$ be the number of stored models, and $r$ be the latent dimension.
    The time complexity of \method is at least $O(m^2 + r^2(d+m))$ and at most $O(\n m^2+m^3+R\n\rank^2(d+m)+\textnormal{\#}\:\!iter\cdot\n\rank^2(d+m))$ per process.
\end{replemma}
\vspace{-1.0em}
\begin{proof}
We analyze the time complexity per time step $t_c$ in the streaming algorithm.
\method first selects the best model for the current window $\stream_c$ via statistical hypothesis testing.
When the test does not trigger a switch, it continues to use the previous best model.
It computes the residual $\delta_{\tc}$ and updates the dictionary inverse $\mK^{-1}$ as shown in \cref{eq:ald,eq:ald_update}, which takes $O(m^2)$.
Subsequently, an online EM algorithm updates the sufficient statistics and parameters.
By using the Woodbury matrix identity (similar to \cref{theorem:time_complexity_of_static_optimization}), the cost is dominated by matrix multiplications involving the latent dimension $r$ and observation dimension $(d+m)$, resulting in $O(r^2(d+m))$.
In contrast, when the test triggers a switch, \method determines to use other model.
\method first takes $O(RT r^2(d+m))$ time to search for the better model in $\regimeset$.
Furthermore, if \method encounters an unknown pattern,
a new model is estimated from scratch using the static optimization on window $\stream_c$.
As proven in \cref{theorem:time_complexity_of_static_optimization}, the cost for this static optimization is $O(\n m^2+m^3+\textnormal{\#}\:\!iter\cdot\n\rank^2(d+m))$.
Finally, forecast an $\ls$-steps-ahead future value via \cref{eq:forecast_latent}, which is negligible compared to the update steps.
Summing these components yields the desired results.
\end{proof}

\section{Experiments}
\label{appendix:experiments}
\subsection{Experimental Setup}
\label{appendix:exp_settings}
\myparaitemize{Computing infrastructure}
We conducted all our experiments on a system running Ubuntu 22.04.4 LTS (GNU/Linux 5.15.0-105-generic x86\_64), equipped with 2 $\ast$ Intel Xeon Gold 6444Y 3.60GHz 16-core CPU, 12 $\ast$ 64GB DDR5 RAM, and 2 $\ast$ 48GB NVIDIA RTX A6000 GPU.

\myparaitemize{Benchmark}
We adopted the dysts benchmark~\cite{Dysts} as mentioned in the main text.
The length of each dataset was $1,000$ steps,
with a time step of $0.01$ and $5\%$ noise ratio.
We normalized the values of all datasets so that the maximum value was $1$ and the minimum value was $-1$.

\myparaitemize{Implementation details}
We now describe the implementation details we used throughout our experiments.
The train/validation/test splits were $20\%/10\%/70\%$, and
we set the length of the current window $\n$ at $100$ steps.
We conducted all our experiments with $5$ different seeds for a fair comparison.
We employed a grid search to determine each optimal hyperparameter such that it minimizes the error over the validation data.
For \method,
we used a standard Gaussian kernel where the kernel width was set to the median distance between the observed data.
We set the threshold $\nu=0.001$, the confidence level $\alpha=0.01$, the switching limit $h=3\chi^2_{\d}(1-\alpha)$, and the number of iterations $\textnormal{\#}\:\!iter=3$.
We varied the forgetting factor $\gamma\in\{0.01, 0.003, 0.001\}$ and the regularization parameter $\lambda_A\in\{10^{-8}, 10^{-7}, 10^{-6}, 10^{-5}\}$.
For ModePlait, we varied the embedding dimension $h\in\{10, 20, 30\}$ and the threshold within $\{1.0, 1.5, 2.0\}$.
For deep learning-based methods, we used the Adam optimizer~\cite{kingma2015method} and varied the learning rate over $\{0.001, 0.003, 0.01, 0.03\}$.

\subsection{Additional Results}
\label{appendix:results}
\subsubsection{Q1. Accuracy}
Here, we provide the detailed results for evaluating \method in terms of real-time forecasting.
Figure \ref{fig:critical_difference} shows the critical difference diagrams for each metric.
This diagram is based on the Wilcoxon-Holm method~\cite{critical-difference}, where methods not connected by a bold line are sufficiently different regarding their average rank.
It is clear that the significant improvements that \method achieved over its baselines are valid according to statistical tests.
Moreover, \cref{detailed_results_20_1,detailed_results_20_2,detailed_results_20_3,detailed_results_25_1,detailed_results_25_2,detailed_results_25_3,detailed_results_30_1,detailed_results_30_2,detailed_results_30_3} show the mean and standard deviation of the results for all $71$ datasets,
where the best and second-best levels of performance are shown in \textbf{bold} and \underline{underlined}, respectively.
We can observe that \method achieved a comprehensive accuracy improvement over all its baselines.

\subsubsection{Q3. Nonlinear capability}
We considered four types of functions: the RBF kernel, the polynomial kernel, the sigmoid kernel, and the linear kernel.
These formulations are given by
\begin{align*}
    k_{\text{RBF}}(\vx,\vy) &=\exp\left(-\frac{|\vx-\vy|^2}{2\sigma^2}\right), \\
    k_{\text{sigmoid}}(\vx,\vy) &= \tanh(\gamma\vx^\top\vy+c_s), \\
    k_{\text{poly}}(\vx,\vy) &= (\gamma\vx^\top\vy + c_p)^d,\\
    k_{\text{linear}}(\vx,\vy) &= \vx^\top \vy,
\end{align*}
where $\sigma > 0$ is the length scale parameter, $\gamma > 0$ is the scale parameter, $c_p, c_s \geq 0$ are constants, and $d \in \N$ is the degree.
We set $\sigma$ as the median of the pairwise distances of the data matrix. For the polynomial and sigmoid kernels, we set $\gamma = 1/D$, where $D$ is the number of dimensions. We used $d=3$ and $c_p=1$ for the polynomial kernel, and $c_s=0$ for the sigmoid kernel.
\cref{fig:sensitivity_analysis} shows critical difference diagrams of \method with various kernel functions for each metric.
Bold lines indicate insignificant differences between connected methods regarding their average rank.
These results statistically demonstrate that \method effectively leverages the RKHS properties to estimate nonlinear dynamics from a nonstationary data stream.

\vspace{0.35em}
\section{Limitations}
While \method achieves remarkable improvements over its baselines, it has some limitations that may be worth addressing in the future.
First, processing tensor streams with \method requires tensor data to be converted into a matricized or vectorized representation. Such a representation does not explicitly preserve the multi-way tensor structure in the model, which may limit its ability to capture multi-aspect relationships.
Second, \method does not explicitly incorporate exogenous inputs or interventions.
Consequently, changes caused by such external factors may be absorbed as regime shifts or process noise rather than being separated from intrinsic mechanisms.
Extending \method toward controlled and intervention-aware dynamical modeling~\cite{Korda2018-hs,Lorch2024-rq} can broaden its applicability to clinical decision-making, closed-loop robotic control, and demand response forecasting.
Lastly, our probabilistic model and change detection mechanism depend on linear Gaussian state space assumptions.
As a result, data streams with extreme outliers or heavy-tailed noise may lead to degraded performance.
Incorporating robust filtering methods based on generalized Bayes~\cite{Duran-Martin2024-bk} or Student's $t$ distribution~\cite{Roth2013-gf} into \method is a promising direction.

\begin{table*}
\centering
\caption{Detailed multivariate forecasting results with forecasting step $\ls=20$}
\vspace{-1.0em}
\label{detailed_results_20_1}
\scalebox{0.94}
{
\begin{tabular}{lllllllll}
\toprule
 & \multicolumn{2}{c}{Aizawa} & \multicolumn{2}{c}{Arneodo} & \multicolumn{2}{c}{Bouali2} & \multicolumn{2}{c}{BurkeShaw} \\
 & \multicolumn{1}{c}{MSE} & \multicolumn{1}{c}{MAE} & \multicolumn{1}{c}{MSE} & \multicolumn{1}{c}{MAE} & \multicolumn{1}{c}{MSE} & \multicolumn{1}{c}{MAE} & \multicolumn{1}{c}{MSE} & \multicolumn{1}{c}{MAE} \\
\midrule
\method & $\mathbf{0.019\pm 0.000}$ & $\mathbf{0.092\pm 0.000}$ & $\mathbf{0.028\pm 0.000}$ & $\mathbf{0.105\pm 0.000}$ & $\mathbf{0.057\pm 0.000}$ & $\mathbf{0.158\pm 0.000}$ & $\mathbf{0.095\pm 0.000}$ & $\mathbf{0.202\pm 0.000}$ \\
ModePlait & $0.309\pm 0.015$ & $0.445\pm 0.009$ & $0.310\pm 0.042$ & $0.417\pm 0.032$ & $0.148\pm 0.002$ & $0.288\pm 0.003$ & $\underline{0.212\pm 0.003}$ & $\underline{0.356\pm 0.002}$ \\
WPMixer & $6.932\pm 13.67$ & $0.876\pm 0.641$ & $0.252\pm 0.230$ & $0.305\pm 0.078$ & $\underline{0.120\pm 0.006}$ & $\underline{0.237\pm 0.003}$ & $0.324\pm 0.328$ & $0.398\pm 0.140$ \\
PAttn & $0.604\pm 0.454$ & $0.585\pm 0.188$ & $\underline{0.142\pm 0.014}$ & $\underline{0.265\pm 0.012}$ & $0.168\pm 0.054$ & $0.279\pm 0.031$ & $0.300\pm 0.127$ & $0.438\pm 0.085$ \\
Koopa & $0.787\pm 0.155$ & $0.629\pm 0.036$ & $0.157\pm 0.012$ & $0.295\pm 0.017$ & $0.141\pm 0.017$ & $0.255\pm 0.013$ & $0.452\pm 0.254$ & $0.489\pm 0.138$ \\
OneNet & $0.426\pm 0.009$ & $0.507\pm 0.004$ & $0.487\pm 0.027$ & $0.570\pm 0.013$ & $0.125\pm 0.010$ & $0.270\pm 0.016$ & $0.236\pm 0.004$ & $0.387\pm 0.005$ \\
\bottomrule
\end{tabular}
}
\end{table*}

\begin{table*}
\centering
\caption{Detailed multivariate forecasting results with forecasting step $\ls=20$}
\vspace{-1.0em}
\label{detailed_results_20_2}
\scalebox{0.94}
{
\begin{tabular}{lllllllll}
\toprule

 & \multicolumn{2}{c}{HyperYan} & \multicolumn{2}{c}{HyperYangChen} & \multicolumn{2}{c}{KawczynskiStrizhak} & \multicolumn{2}{c}{Laser} \\
 & \multicolumn{1}{c}{MSE} & \multicolumn{1}{c}{MAE} & \multicolumn{1}{c}{MSE} & \multicolumn{1}{c}{MAE} & \multicolumn{1}{c}{MSE} & \multicolumn{1}{c}{MAE} & \multicolumn{1}{c}{MSE} & \multicolumn{1}{c}{MAE} \\
\midrule
\method & $\mathbf{0.069\pm 0.000}$ & $\mathbf{0.198\pm 0.000}$ & $\mathbf{0.117\pm 0.000}$ & $\mathbf{0.252\pm 0.000}$ & $\mathbf{0.034\pm 0.000}$ & $\mathbf{0.111\pm 0.000}$ & $\mathbf{0.083\pm 0.000}$ & $\mathbf{0.197\pm 0.000}$ \\
ModePlait & $0.227\pm 0.006$ & $\underline{0.367\pm 0.008}$ & $0.288\pm 0.007$ & $0.445\pm 0.005$ & $\underline{0.037\pm 0.000}$ & $0.117\pm 0.002$ & $\underline{0.324\pm 0.009}$ & $\underline{0.447\pm 0.007}$ \\
WPMixer & $0.311\pm 0.172$ & $0.408\pm 0.084$ & $\underline{0.227\pm 0.039}$ & $\underline{0.386\pm 0.032}$ & $0.039\pm 0.002$ & $\underline{0.111\pm 0.003}$ & $0.503\pm 0.357$ & $0.496\pm 0.148$ \\
PAttn & $0.333\pm 0.137$ & $0.449\pm 0.105$ & $0.363\pm 0.080$ & $0.489\pm 0.041$ & $0.044\pm 0.002$ & $0.119\pm 0.004$ & $0.401\pm 0.144$ & $0.492\pm 0.064$ \\
Koopa & $0.349\pm 0.081$ & $0.448\pm 0.059$ & $0.352\pm 0.044$ & $0.477\pm 0.033$ & $0.045\pm 0.001$ & $0.119\pm 0.002$ & $0.393\pm 0.295$ & $0.474\pm 0.138$ \\
OneNet & $\underline{0.224\pm 0.056}$ & $0.369\pm 0.052$ & $0.258\pm 0.016$ & $0.426\pm 0.014$ & $0.046\pm 0.001$ & $0.133\pm 0.002$ & $0.360\pm 0.005$ & $0.477\pm 0.004$ \\
\bottomrule
\end{tabular}
}
\end{table*}

\begin{table*}
\centering
\caption{Detailed multivariate forecasting results with forecasting step $\ls=20$}
\vspace{-1.0em}
\label{detailed_results_20_3}
\scalebox{0.94}
{
\begin{tabular}{lllllllll}
\toprule

 & \multicolumn{2}{c}{SprottB} & \multicolumn{2}{c}{SprottC} & \multicolumn{2}{c}{SprottD} & \multicolumn{2}{c}{SprottE} \\
 & \multicolumn{1}{c}{MSE} & \multicolumn{1}{c}{MAE} & \multicolumn{1}{c}{MSE} & \multicolumn{1}{c}{MAE} & \multicolumn{1}{c}{MSE} & \multicolumn{1}{c}{MAE} & \multicolumn{1}{c}{MSE} & \multicolumn{1}{c}{MAE} \\
\midrule
\method & $\mathbf{0.075\pm 0.000}$ & $\mathbf{0.191\pm 0.000}$ & $\mathbf{0.077\pm 0.000}$ & $\mathbf{0.176\pm 0.000}$ & $\mathbf{0.041\pm 0.000}$ & $\mathbf{0.150\pm 0.000}$ & $\mathbf{0.122\pm 0.000}$ & $\mathbf{0.240\pm 0.000}$ \\
ModePlait & $0.288\pm 0.025$ & $0.433\pm 0.022$ & $0.233\pm 0.020$ & $0.352\pm 0.013$ & $0.187\pm 0.006$ & $0.330\pm 0.008$ & $0.291\pm 0.005$ & $0.417\pm 0.008$ \\
WPMixer & $1.089\pm 0.486$ & $0.689\pm 0.133$ & $0.219\pm 0.076$ & $0.344\pm 0.053$ & $\underline{0.132\pm 0.018}$ & $\underline{0.274\pm 0.019}$ & $1.023\pm 1.118$ & $0.560\pm 0.239$ \\
PAttn & $0.365\pm 0.036$ & $0.492\pm 0.024$ & $0.298\pm 0.111$ & $0.401\pm 0.073$ & $0.211\pm 0.059$ & $0.339\pm 0.042$ & $0.648\pm 0.242$ & $0.605\pm 0.129$ \\
Koopa & $0.617\pm 0.320$ & $0.577\pm 0.113$ & $0.444\pm 0.181$ & $0.447\pm 0.093$ & $0.281\pm 0.084$ & $0.377\pm 0.055$ & $0.597\pm 0.317$ & $0.537\pm 0.134$ \\
OneNet & $\underline{0.274\pm 0.017}$ & $\underline{0.409\pm 0.012}$ & $\underline{0.177\pm 0.003}$ & $\underline{0.302\pm 0.005}$ & $0.159\pm 0.002$ & $0.316\pm 0.002$ & $\underline{0.238\pm 0.038}$ & $0.386\pm 0.027$ \\
\bottomrule
\end{tabular}
}
\end{table*}

\begin{table*}
\centering
\caption{Detailed multivariate forecasting results with forecasting step $\ls=25$}
\vspace{-1.0em}
\label{detailed_results_25_1}
\scalebox{0.94}
{
\begin{tabular}{lllllllll}
\toprule
 & \multicolumn{2}{c}{Aizawa} & \multicolumn{2}{c}{Arneodo} & \multicolumn{2}{c}{Bouali2} & \multicolumn{2}{c}{BurkeShaw} \\
 & \multicolumn{1}{c}{MSE} & \multicolumn{1}{c}{MAE} & \multicolumn{1}{c}{MSE} & \multicolumn{1}{c}{MAE} & \multicolumn{1}{c}{MSE} & \multicolumn{1}{c}{MAE} & \multicolumn{1}{c}{MSE} & \multicolumn{1}{c}{MAE} \\
\midrule
\method & $\mathbf{0.027\pm 0.000}$ & $\mathbf{0.112\pm 0.000}$ & $\mathbf{0.035\pm 0.000}$ & $\mathbf{0.122\pm 0.000}$ & $\mathbf{0.073\pm 0.000}$ & $\mathbf{0.181\pm 0.000}$ & $\mathbf{0.142\pm 0.000}$ & $\mathbf{0.268\pm 0.000}$ \\
ModePlait & $0.342\pm 0.018$ & $0.462\pm 0.010$ & $0.336\pm 0.031$ & $0.437\pm 0.022$ & $0.153\pm 0.003$ & $0.295\pm 0.002$ & $\underline{0.227\pm 0.002}$ & $\underline{0.368\pm 0.001}$ \\
WPMixer & $5.689\pm 10.74$ & $0.881\pm 0.599$ & $0.266\pm 0.228$ & $0.322\pm 0.076$ & $0.125\pm 0.003$ & $\underline{0.244\pm 0.006}$ & $0.271\pm 0.209$ & $0.384\pm 0.097$ \\
PAttn & $0.593\pm 0.341$ & $0.590\pm 0.153$ & $0.159\pm 0.009$ & $\underline{0.279\pm 0.006}$ & $0.183\pm 0.047$ & $0.284\pm 0.025$ & $0.320\pm 0.116$ & $0.458\pm 0.078$ \\
Koopa & $0.824\pm 0.076$ & $0.666\pm 0.024$ & $\underline{0.159\pm 0.017}$ & $0.293\pm 0.013$ & $0.150\pm 0.006$ & $0.267\pm 0.004$ & $0.470\pm 0.275$ & $0.501\pm 0.152$ \\
OneNet & $0.432\pm 0.014$ & $0.510\pm 0.007$ & $0.497\pm 0.025$ & $0.577\pm 0.009$ & $\underline{0.125\pm 0.008}$ & $0.267\pm 0.009$ & $0.240\pm 0.001$ & $0.391\pm 0.002$ \\
\bottomrule
\end{tabular}
}
\end{table*}

\begin{table*}
\centering
\caption{Detailed multivariate forecasting results with forecasting step $\ls=25$}
\vspace{-1.0em}
\label{detailed_results_25_2}
\scalebox{0.94}
{
\begin{tabular}{lllllllll}
\toprule

 & \multicolumn{2}{c}{HyperYan} & \multicolumn{2}{c}{HyperYangChen} & \multicolumn{2}{c}{KawczynskiStrizhak} & \multicolumn{2}{c}{Laser} \\
 & \multicolumn{1}{c}{MSE} & \multicolumn{1}{c}{MAE} & \multicolumn{1}{c}{MSE} & \multicolumn{1}{c}{MAE} & \multicolumn{1}{c}{MSE} & \multicolumn{1}{c}{MAE} & \multicolumn{1}{c}{MSE} & \multicolumn{1}{c}{MAE} \\
\midrule
\method & $\mathbf{0.118\pm 0.000}$ & $\mathbf{0.261\pm 0.000}$ & $\mathbf{0.178\pm 0.000}$ & $\mathbf{0.319\pm 0.000}$ & $\mathbf{0.046\pm 0.000}$ & $\underline{0.121\pm 0.000}$ & $\mathbf{0.125\pm 0.000}$ & $\mathbf{0.248\pm 0.000}$ \\
ModePlait & $0.241\pm 0.006$ & $0.382\pm 0.005$ & $0.310\pm 0.006$ & $0.461\pm 0.007$ & $\underline{0.050\pm 0.000}$ & $0.129\pm 0.002$ & $0.345\pm 0.009$ & $0.468\pm 0.007$ \\
WPMixer & $0.270\pm 0.088$ & $0.395\pm 0.058$ & $\underline{0.256\pm 0.038}$ & $\underline{0.411\pm 0.029}$ & $0.051\pm 0.003$ & $\mathbf{0.121\pm 0.003}$ & $0.508\pm 0.376$ & $0.511\pm 0.160$ \\
PAttn & $0.322\pm 0.116$ & $0.445\pm 0.095$ & $0.377\pm 0.059$ & $0.502\pm 0.027$ & $0.058\pm 0.004$ & $0.127\pm 0.008$ & $0.409\pm 0.116$ & $0.505\pm 0.065$ \\
Koopa & $0.335\pm 0.074$ & $0.443\pm 0.048$ & $0.376\pm 0.041$ & $0.491\pm 0.029$ & $0.056\pm 0.002$ & $0.125\pm 0.002$ & $\underline{0.341\pm 0.203}$ & $\underline{0.448\pm 0.105}$ \\
OneNet & $\underline{0.233\pm 0.058}$ & $\underline{0.377\pm 0.053}$ & $0.263\pm 0.024$ & $0.429\pm 0.021$ & $0.058\pm 0.000$ & $0.142\pm 0.001$ & $0.372\pm 0.003$ & $0.489\pm 0.003$ \\
\bottomrule
\end{tabular}
}
\end{table*}

\begin{table*}
\centering
\caption{Detailed multivariate forecasting results with forecasting step $\ls=25$}
\vspace{-1.0em}
\label{detailed_results_25_3}
\scalebox{0.94}
{
\begin{tabular}{lllllllll}
\toprule

 & \multicolumn{2}{c}{SprottB} & \multicolumn{2}{c}{SprottC} & \multicolumn{2}{c}{SprottD} & \multicolumn{2}{c}{SprottE} \\
 & \multicolumn{1}{c}{MSE} & \multicolumn{1}{c}{MAE} & \multicolumn{1}{c}{MSE} & \multicolumn{1}{c}{MAE} & \multicolumn{1}{c}{MSE} & \multicolumn{1}{c}{MAE} & \multicolumn{1}{c}{MSE} & \multicolumn{1}{c}{MAE} \\
\midrule
\method & $\mathbf{0.112\pm 0.000}$ & $\mathbf{0.240\pm 0.000}$ & $\mathbf{0.107\pm 0.000}$ & $\mathbf{0.215\pm 0.000}$ & $\mathbf{0.062\pm 0.000}$ & $\mathbf{0.186\pm 0.000}$ & $\mathbf{0.164\pm 0.000}$ & $\mathbf{0.288\pm 0.000}$ \\
ModePlait & $0.300\pm 0.022$ & $0.444\pm 0.018$ & $0.242\pm 0.020$ & $0.361\pm 0.015$ & $0.210\pm 0.011$ & $0.348\pm 0.010$ & $0.293\pm 0.003$ & $0.419\pm 0.008$ \\
WPMixer & $0.854\pm 0.384$ & $0.647\pm 0.139$ & $0.232\pm 0.072$ & $0.357\pm 0.050$ & $\underline{0.168\pm 0.012}$ & $\underline{0.308\pm 0.011}$ & $0.755\pm 0.701$ & $0.544\pm 0.205$ \\
PAttn & $0.353\pm 0.040$ & $0.474\pm 0.028$ & $0.323\pm 0.116$ & $0.417\pm 0.072$ & $0.244\pm 0.072$ & $0.365\pm 0.051$ & $0.622\pm 0.228$ & $0.600\pm 0.120$ \\
Koopa & $0.555\pm 0.276$ & $0.549\pm 0.102$ & $0.396\pm 0.123$ & $0.440\pm 0.078$ & $0.274\pm 0.042$ & $0.378\pm 0.020$ & $0.806\pm 0.573$ & $0.551\pm 0.151$ \\
OneNet & $\underline{0.277\pm 0.015}$ & $\underline{0.412\pm 0.010}$ & $\underline{0.179\pm 0.007}$ & $\underline{0.302\pm 0.010}$ & $0.170\pm 0.002$ & $0.327\pm 0.004$ & $\underline{0.243\pm 0.030}$ & $\underline{0.387\pm 0.020}$ \\
\bottomrule
\end{tabular}
}
\end{table*}

\begin{table*}
\centering
\caption{Detailed multivariate forecasting results with forecasting step $\ls=30$}
\vspace{-1.0em}
\label{detailed_results_30_1}
\scalebox{0.94}
{
\begin{tabular}{lllllllll}
\toprule
 & \multicolumn{2}{c}{Aizawa} & \multicolumn{2}{c}{Arneodo} & \multicolumn{2}{c}{Bouali2} & \multicolumn{2}{c}{BurkeShaw} \\
 & \multicolumn{1}{c}{MSE} & \multicolumn{1}{c}{MAE} & \multicolumn{1}{c}{MSE} & \multicolumn{1}{c}{MAE} & \multicolumn{1}{c}{MSE} & \multicolumn{1}{c}{MAE} & \multicolumn{1}{c}{MSE} & \multicolumn{1}{c}{MAE} \\
\midrule
\method & $\mathbf{0.039\pm 0.000}$ & $\mathbf{0.136\pm 0.000}$ & $\mathbf{0.044\pm 0.000}$ & $\mathbf{0.141\pm 0.000}$ & $\mathbf{0.100\pm 0.000}$ & $\mathbf{0.213\pm 0.000}$ & $\mathbf{0.180\pm 0.000}$ & $\mathbf{0.309\pm 0.000}$ \\
ModePlait & $\underline{0.395\pm 0.018}$ & $0.489\pm 0.010$ & $0.369\pm 0.025$ & $0.463\pm 0.017$ & $0.162\pm 0.005$ & $0.306\pm 0.002$ & $0.244\pm 0.004$ & $0.384\pm 0.003$ \\
WPMixer & $3.188\pm 5.154$ & $0.845\pm 0.453$ & $0.351\pm 0.375$ & $0.350\pm 0.099$ & $0.129\pm 0.009$ & $\underline{0.244\pm 0.006}$ & $\underline{0.228\pm 0.110}$ & $\underline{0.370\pm 0.063}$ \\
PAttn & $0.543\pm 0.309$ & $0.560\pm 0.143$ & $0.168\pm 0.013$ & $0.289\pm 0.012$ & $0.194\pm 0.057$ & $0.291\pm 0.034$ & $0.312\pm 0.077$ & $0.445\pm 0.048$ \\
Koopa & $0.887\pm 0.137$ & $0.699\pm 0.043$ & $\underline{0.167\pm 0.012}$ & $0.297\pm 0.017$ & $0.161\pm 0.018$ & $0.276\pm 0.016$ & $0.447\pm 0.239$ & $0.489\pm 0.136$ \\
OneNet & $0.431\pm 0.011$ & $0.509\pm 0.005$ & $0.509\pm 0.024$ & $0.585\pm 0.010$ & $\underline{0.122\pm 0.007}$ & $0.268\pm 0.008$ & $0.237\pm 0.002$ & $0.388\pm 0.002$ \\
\bottomrule
\end{tabular}
}
\end{table*}

\begin{table*}
\centering
\caption{Detailed multivariate forecasting results with forecasting step $\ls=30$}
\vspace{-1.0em}
\label{detailed_results_30_2}
\scalebox{0.94}
{
\begin{tabular}{lllllllll}
\toprule

 & \multicolumn{2}{c}{HyperYan} & \multicolumn{2}{c}{HyperYangChen} & \multicolumn{2}{c}{KawczynskiStrizhak} & \multicolumn{2}{c}{Laser} \\
 & \multicolumn{1}{c}{MSE} & \multicolumn{1}{c}{MAE} & \multicolumn{1}{c}{MSE} & \multicolumn{1}{c}{MAE} & \multicolumn{1}{c}{MSE} & \multicolumn{1}{c}{MAE} & \multicolumn{1}{c}{MSE} & \multicolumn{1}{c}{MAE} \\
\midrule
\method & $\mathbf{0.144\pm 0.000}$ & $\mathbf{0.289\pm 0.000}$ & $\mathbf{0.215\pm 0.000}$ & $\mathbf{0.357\pm 0.000}$ & $\mathbf{0.059\pm 0.000}$ & $\underline{0.132\pm 0.000}$ & $\mathbf{0.179\pm 0.000}$ & $\mathbf{0.304\pm 0.000}$ \\
ModePlait & $0.259\pm 0.011$ & $0.395\pm 0.007$ & $0.334\pm 0.010$ & $0.478\pm 0.011$ & $\underline{0.064\pm 0.000}$ & $0.143\pm 0.002$ & $0.371\pm 0.009$ & $0.495\pm 0.008$ \\
WPMixer & $0.261\pm 0.068$ & $0.394\pm 0.050$ & $0.286\pm 0.041$ & $0.435\pm 0.030$ & $0.065\pm 0.002$ & $\mathbf{0.131\pm 0.002}$ & $0.513\pm 0.425$ & $0.520\pm 0.174$ \\
PAttn & $0.327\pm 0.120$ & $0.445\pm 0.090$ & $0.407\pm 0.077$ & $0.520\pm 0.034$ & $0.070\pm 0.003$ & $0.135\pm 0.003$ & $0.414\pm 0.108$ & $0.512\pm 0.067$ \\
Koopa & $0.297\pm 0.038$ & $0.417\pm 0.029$ & $0.382\pm 0.024$ & $0.502\pm 0.014$ & $0.071\pm 0.004$ & $0.137\pm 0.006$ & $\underline{0.339\pm 0.150}$ & $\underline{0.454\pm 0.066}$ \\
OneNet & $\underline{0.236\pm 0.065}$ & $\underline{0.377\pm 0.058}$ & $\underline{0.260\pm 0.022}$ & $\underline{0.421\pm 0.021}$ & $0.069\pm 0.001$ & $0.152\pm 0.002$ & $0.385\pm 0.003$ & $0.501\pm 0.003$ \\
\bottomrule
\end{tabular}
}
\end{table*}

\begin{table*}
\centering
\caption{Detailed multivariate forecasting results with forecasting step $\ls=30$}
\vspace{-1.0em}
\label{detailed_results_30_3}
\scalebox{0.94}
{
\begin{tabular}{lllllllll}
\toprule

 & \multicolumn{2}{c}{SprottB} & \multicolumn{2}{c}{SprottC} & \multicolumn{2}{c}{SprottD} & \multicolumn{2}{c}{SprottE} \\
 & \multicolumn{1}{c}{MSE} & \multicolumn{1}{c}{MAE} & \multicolumn{1}{c}{MSE} & \multicolumn{1}{c}{MAE} & \multicolumn{1}{c}{MSE} & \multicolumn{1}{c}{MAE} & \multicolumn{1}{c}{MSE} & \multicolumn{1}{c}{MAE} \\
\midrule
\method & $\mathbf{0.123\pm 0.000}$ & $\mathbf{0.255\pm 0.000}$ & $\mathbf{0.141\pm 0.000}$ & $\mathbf{0.252\pm 0.000}$ & $\mathbf{0.083\pm 0.000}$ & $\mathbf{0.221\pm 0.000}$ & $\mathbf{0.186\pm 0.000}$ & $\mathbf{0.318\pm 0.000}$ \\
ModePlait & $0.314\pm 0.024$ & $0.454\pm 0.014$ & $0.254\pm 0.019$ & $0.371\pm 0.015$ & $0.235\pm 0.015$ & $0.366\pm 0.010$ & $0.282\pm 0.004$ & $0.416\pm 0.005$ \\
WPMixer & $0.675\pm 0.293$ & $0.610\pm 0.123$ & $0.245\pm 0.081$ & $0.363\pm 0.050$ & $0.195\pm 0.015$ & $\underline{0.327\pm 0.010}$ & $0.617\pm 0.519$ & $0.527\pm 0.171$ \\
PAttn & $0.337\pm 0.044$ & $0.472\pm 0.026$ & $0.352\pm 0.134$ & $0.436\pm 0.078$ & $0.270\pm 0.065$ & $0.389\pm 0.044$ & $0.543\pm 0.175$ & $0.570\pm 0.105$ \\
Koopa & $0.544\pm 0.257$ & $0.544\pm 0.100$ & $0.429\pm 0.117$ & $0.456\pm 0.071$ & $0.286\pm 0.043$ & $0.388\pm 0.018$ & $0.595\pm 0.316$ & $0.521\pm 0.108$ \\
OneNet & $\underline{0.272\pm 0.014}$ & $\underline{0.409\pm 0.008}$ & $\underline{0.183\pm 0.005}$ & $\underline{0.302\pm 0.008}$ & $\underline{0.179\pm 0.001}$ & $0.335\pm 0.002$ & $\underline{0.239\pm 0.031}$ & $\underline{0.381\pm 0.026}$ \\
\bottomrule
\end{tabular}
}
\end{table*}

\end{document}